\documentclass{article}

\usepackage{arxiv}

\usepackage[utf8]{inputenc} 
\usepackage[T1]{fontenc}    
\usepackage{hyperref}       
\usepackage{url}            
\usepackage{booktabs}       
\usepackage{amsfonts}       
\usepackage{nicefrac}       
\usepackage{microtype}      
\usepackage{lipsum}
\usepackage{graphicx}
\usepackage{algorithm}
\usepackage{algorithmic}
\usepackage{amsmath}
\usepackage{bm}
\usepackage{graphicx}
\graphicspath{{figures/}{./}}
\usepackage{subcaption}

\usepackage{newfloat}
\usepackage{listings}
\usepackage{bm}

\floatstyle{ruled}
\newfloat{listing}{tb}{lst}{}
\floatname{listing}{Listing}

\usepackage{booktabs}
\usepackage{amssymb}

\newtheorem{theorem}{Theorem}
\newtheorem{corollary}{Corollary}

\newtheorem{proof}{Proof}

\newcommand{\R}{\mathbb{R}}

\newcommand{\E}{\mathcal{E}}

\newcommand{\norm}[1]{\left\lVert #1 \right\rVert}

\newcommand{\Phih}{\widehat{\Phi}_{\theta}}
\newcommand{\At}{\widetilde{A}}
\newcommand{\Bt}{\widetilde{B}}
\newcommand{\Ht}{\widetilde{H}}
\newcommand{\Q}{\mathcal{Q}}

\title{CoSynFlow: Conformal Symplectic Neural Flows \\ for Cross-System Prediction of \\Dissipative Hamiltonian Dynamics}

\author{
 Baige Xu \\
  RIKEN AIP\\
	744 Motooka Nishi-ku, Fukuoka, Japan\\
  \texttt{baige.xu@riken.jp} \\
   \And
 Takaharu Yaguchi \\
  IMI, Kyushu University\\
    RIKEN AIP\\
	744 Motooka Nishi-ku, Fukuoka, Japan\\
  \texttt{yaguchi@imi.kyushu-u.ac.jp} \\
}
\renewcommand{\today}{}
\begin{document}
\fancyhead[LE,RO]{}

\maketitle

\begin{abstract}
Learning solution operators for differential equations is a central
problem in scientific machine learning. However, many neural operator
methods optimize prediction accuracy without explicitly enforcing the
geometric structure of the dynamics. Structure-preserving models such
as SympNets and Symplectic Neural Flows address this issue for
conservative Hamiltonian systems by preserving the symplectic form. In
dissipative Hamiltonian systems with conformal symplectic structure,
however, the symplectic form evolves according to a conformal factor
determined by the dissipation.
We propose \textbf{CoSynFlow}, a conformal symplectic neural flow for
learning continuous-time solution maps of dissipative Hamiltonian
dynamics. CoSynFlow composes symplectic shear maps with explicit
conformal scaling, preserving the conformal symplectic structure by
construction. By conditioning it on a finite-dimensional Hamiltonian
descriptor and the dissipation parameter, a single trained model
predicts solution maps for unseen systems without retraining.
CoSynFlow keeps the structure error at machine precision, attains the
lowest long-horizon error, and admits physics-informed training.
\end{abstract}

\section{Introduction}
\label{sec:introduction}
Learning solution maps of differential equations is a central goal of
scientific machine learning. Neural operators are a natural framework
for this task, since one trained model represents solutions across
varying initial conditions, coefficients, or governing
equations~\cite{deeponet,fno,no}. However, they treat the solution map
as an unconstrained function approximation, with no guarantee that the
predicted dynamics respect the geometric structure of the underlying
equations.

Geometric structure is fundamental to Hamiltonian dynamics. The flow of
a conservative Hamiltonian system preserves the symplectic form, and
therefore also the phase-space volume. Hamiltonian Neural Networks,
SympNets, and Symplectic Neural Flows encode this geometry at the level of learned vector fields, maps, or flow architectures~\cite{hnn,sympnets,snf}.
However, existing map- and flow-based architectures are designed around ordinary symplecticity and are commonly trained for a fixed Hamiltonian system.

The setting considered here raises three challenges. First, ordinary
symplectic preservation is not the appropriate constraint for
dissipative Hamiltonian dynamics. A broad class of such systems can be
formulated as conformal Hamiltonian systems~\cite{conformal}, whose vector field $X$ and flow $\phi_t$, the solution map at time $t$, satisfy
\begin{align*}
    \mathcal{L}_{X}\omega = \gamma\omega,
    \qquad
    \phi_t^*\omega = e^{\gamma t}\omega,
\end{align*}
where $\gamma<0$ is the dissipation parameter. The symplectic form is
scaled by a prescribed conformal factor rather than preserved exactly.
Second, conventional integrators and learned fixed-step maps reach a
target time through repeated local updates, which motivates a single
model that represents the solution maps over a continuous range of
query times. Third, a structure-preserving model trained for one
Hamiltonian does not provide a shared solution model for other
Hamiltonians, and applying such models to a system family requires
separate training for each system.

We propose \textbf{CoSynFlow}, a system-conditioned conformal
symplectic neural flow for learning continuous-time solution maps of
dissipative Hamiltonian dynamics. Each time-$t$ map is a composition of
symmetric blocks built from gradient shear maps and explicit conformal
scaling. The shears are symplectic because they are generated by scalar
potentials, and the layer-wise scaling factors are normalized to
multiply to $e^{\gamma t}$, so the map is exactly conformal symplectic
for arbitrary trainable parameters. It takes the query time as a
continuous input rather than a fixed step. An operator-learning
conditioner maps a finite-dimensional Hamiltonian descriptor and the
dissipation parameter to the block parameters, so one trained model
predicts the solution map of an unseen system without retraining.
We train CoSynFlow on randomly generated Hamiltonian systems and
evaluate it on held-out systems from the training distribution and on
four analytic benchmark systems outside it, across several dissipation
rates, and beyond the training interval by composing the learned maps. We compare it with an ordinary symplectic neural flow,
unstructured and softly constrained direct-flow models, and
representative neural operator architectures, reporting predictive
accuracy, energy behavior, and the conformal symplectic structure error.

The main contributions of this work are as follows:
\begin{itemize}
    \item \textbf{Exact conformal symplecticity.}
    For any trainable parameters, the resulting time-$t$ map is exactly
    conformal symplectic by construction, matching the geometry of
    dissipative Hamiltonian dynamics.

    \item \textbf{Continuous-time flow learning.}
    The model produces the time-$t$ map for any $t$ in a single evaluation
    and composes to predict beyond the training horizon with the exact
    conformal factor. Differentiability in $t$ also admits physics-informed
    training on the residual of the governing equation, which is undefined
    for a fixed-step map.

    \item \textbf{Cross-system prediction.}
    One shared model predicts the solution map of an unseen dissipative
    Hamiltonian system without retraining, and conditioning leaves the exact
    structure intact.
\end{itemize}

\section{Related Work}

\paragraph{Neural Operators for Solution-Map Learning.}
Deep Operator Networks (DeepONets), Fourier Neural Operators (FNOs),
Physics-Informed Neural Operators (PINOs) and Transolver learn maps
between function spaces and predict across varying input functions and
system parameters~\cite{deeponet,fno,pino,transolver}. None of them
imposes Hamiltonian geometry on the learned operator by construction.
CoSynFlow instead uses operator learning to condition a finite-time flow whose geometry is fixed by the architecture.

\paragraph{Structure-Preserving Learning for Hamiltonian Dynamics.}
Hamiltonian Neural Networks (HNNs) and Lagrangian Neural Networks
(LNNs) learn scalar Hamiltonian or Lagrangian functions and derive the
dynamics from them, while Symplectic ODE-Net (SymODEN) incorporates
control into a Hamiltonian vector-field model~\cite{hnn,lnn,symoden}.
SympNets parameterize exactly symplectic evolution maps, and Symplectic
Neural Flows (SympFlow) construct time-dependent symplectic flow
maps~\cite{sympnets,snf}. Poisson Neural Networks and Neural Symplectic
Form extend structure-preserving learning to Poisson systems and
noncanonical coordinates~\cite{pnn,nsf}. These methods address
conservative dynamics. CoSynFlow instead targets dissipative systems,
whose flows scale rather than preserve the symplectic form.

\paragraph{Learning Dissipative Hamiltonian Dynamics.}
Dissipative Hamiltonian Neural Networks (D-HNNs) separate Hamiltonian
and Rayleigh-dissipative components, while Dissipative SymODEN and
port-Hamiltonian neural networks encode dissipation and external inputs
in structured continuous-time dynamics~\cite{dhnn,dsymoden,phnn}.
Discrete-gradient models and energy-behavior-preserving integrators
enforce conservation or dissipation laws in discrete time, and
GENERIC-informed networks encode thermodynamic
consistency~\cite{dgnet,uepi,gfinns}. These approaches identify unknown
dynamics from trajectory data, whereas CoSynFlow learns a shared family
of solution maps for systems that are specified, and predicts unseen
systems without retraining.

\section{Preliminaries and Problem Setting}
\label{sec:preliminaries}

\subsection{Conformal Symplectic Structure}
\label{sec:conformal_symplectic_structure}

Let $(\mathcal{M},\omega)$ be an exact symplectic manifold, so that
$\omega=-\mathrm{d}\theta$ for a one-form $\theta$. For a sufficiently
smooth Hamiltonian $H:\mathcal{M}\rightarrow\mathbb{R}$, the Hamiltonian
vector field $X_H$ is defined by $\iota_{X_H}\omega=\mathrm{d}H$, and
Cartan's formula together with $\mathrm{d}\omega=0$ gives
$\mathcal{L}_{X_H}\omega=0$. The Liouville vector field $Z$ is defined
by $\iota_{Z}\omega=-\theta$ and satisfies
$\mathcal{L}_{Z}\omega=\omega$ for the same reason. We consider vector
fields of the form
\begin{align}
    X &= X_H+\gamma Z, \qquad \gamma<0,
    \label{eq:conformal_vector_decomposition}
\end{align}
which by linearity of the Lie derivative satisfy
$\mathcal{L}_{X}\omega=\gamma\omega$ and are called conformal
symplectic. Let $\phi_t$ denote the flow of $X$, so that
$z(t)=\phi_t(z_0)$ with $z(0)=z_0$. Since
$\frac{\mathrm{d}}{\mathrm{d}t}\phi_t^{*}\omega
=\phi_t^{*}\mathcal{L}_{X}\omega=\gamma\,\phi_t^{*}\omega$ and
$\phi_0^{*}\omega=\omega$,
\begin{align}
    \phi_t^{*}\omega &= e^{\gamma t}\omega .
    \label{eq:true_flow_conformal_symplectic}
\end{align}
For $\gamma<0$ the symplectic form contracts by the prescribed factor
$e^{\gamma t}$. This is the geometric relation imposed on the learned
maps.

Throughout this work we take $\mathcal{M}=\mathbb{R}^{2d}$ with its
canonical structure, writing $z=(q,p)$ with $q,p\in\mathbb{R}^{d}$ the
generalized positions and momenta, and
$J=\left(\begin{smallmatrix}0&I_d\\-I_d&0\end{smallmatrix}\right)$. Then
$\theta=\sum_{i=1}^{d}p_i\,\mathrm{d}q_i$,
$\omega=\sum_{i=1}^{d}\mathrm{d}q_i\wedge\mathrm{d}p_i$, $X_H=J\nabla H$
and $Z(q,p)=(0,p)$, and \eqref{eq:true_flow_conformal_symplectic} reads
$D\phi_t(z)^{\top}J\,D\phi_t(z)=e^{\gamma t}J$.

\subsection{Dissipative Hamiltonian Systems}
In canonical coordinates, \eqref{eq:conformal_vector_decomposition} takes
the form
\begin{align}
    \dot z &= J\nabla H(z) + \gamma
    \begin{pmatrix} 0 & 0\\ 0 & I_d\end{pmatrix} z,
    \label{eq:dissipative_hamiltonian}
\end{align}
or equivalently $\dot q=\nabla_pH(q,p)$ and
$\dot p=-\nabla_qH(q,p)+\gamma p$. Along solutions
$\mathrm{d}H/\mathrm{d}t=\gamma\,p^\top\nabla_pH(q,p)$, so the energy is
no longer conserved and its instantaneous rate of change is proportional
to $\gamma$. The dissipation acts on the momentum variables alone, and
$\gamma=0$ recovers the conservative case. The solution flow of
\eqref{eq:dissipative_hamiltonian} is $\phi_t$ and satisfies
\eqref{eq:true_flow_conformal_symplectic}.

\subsection{Cross-System Solution-Map Learning}
\label{sec:cross_system_problem}

Let $\mathcal{H}$ be a family of sufficiently smooth Hamiltonians on
$\mathbb{R}^{2d}$ and let $\Gamma\subset(-\infty,0)$ be a set of
dissipation parameters. Each pair $(H,\gamma)\in\mathcal{H}\times\Gamma$
defines a system of the form \eqref{eq:dissipative_hamiltonian}, whose
solution flow $\phi_{H,\gamma}^{t}$ satisfies
$(\phi_{H,\gamma}^{t})^*\omega=e^{\gamma t}\omega$.

In the cross-system setting the Hamiltonian is not fixed, and we assume
that $H$ and $\gamma$ are known for each system. To condition a shared
model on the Hamiltonian, let $r_1,\ldots,r_m\in\mathbb{R}^{2d}$ be
fixed sensor points and define the finite-dimensional descriptor
$h_H=\mathcal{E}(H):=(H(r_1),\ldots,H(r_m))\in\mathbb{R}^{m}$ and the
system condition $\xi:=(h_H,\gamma)$. The known Hamiltonian enters the
structure-preserving construction directly, while $h_H$ conditions its
shared trainable components.

The objective is to learn one model for the entire family. Given an
initial state $z_0$, a query time $t\in[0,T]$, a known Hamiltonian $H$
and the system condition $\xi$, the model approximates
\begin{align*}
    \widehat{\Phi}_{\theta}^{t}(z_0;H,\xi)
    \approx \phi_{H,\gamma}^{t}(z_0),
\end{align*}
so that a single set of trainable parameters represents finite-time
solution maps for different Hamiltonians and dissipation parameters,
rather than a separate model for each system. In addition to predictive
accuracy, every conditioned map is required to satisfy
$(\widehat{\Phi}_{\theta}^{t}(\,\cdot\,;H,\xi))^*\omega=e^{\gamma t}\omega$.

\begin{figure*}[t]
    \centering
    \includegraphics[width=0.8\textwidth]{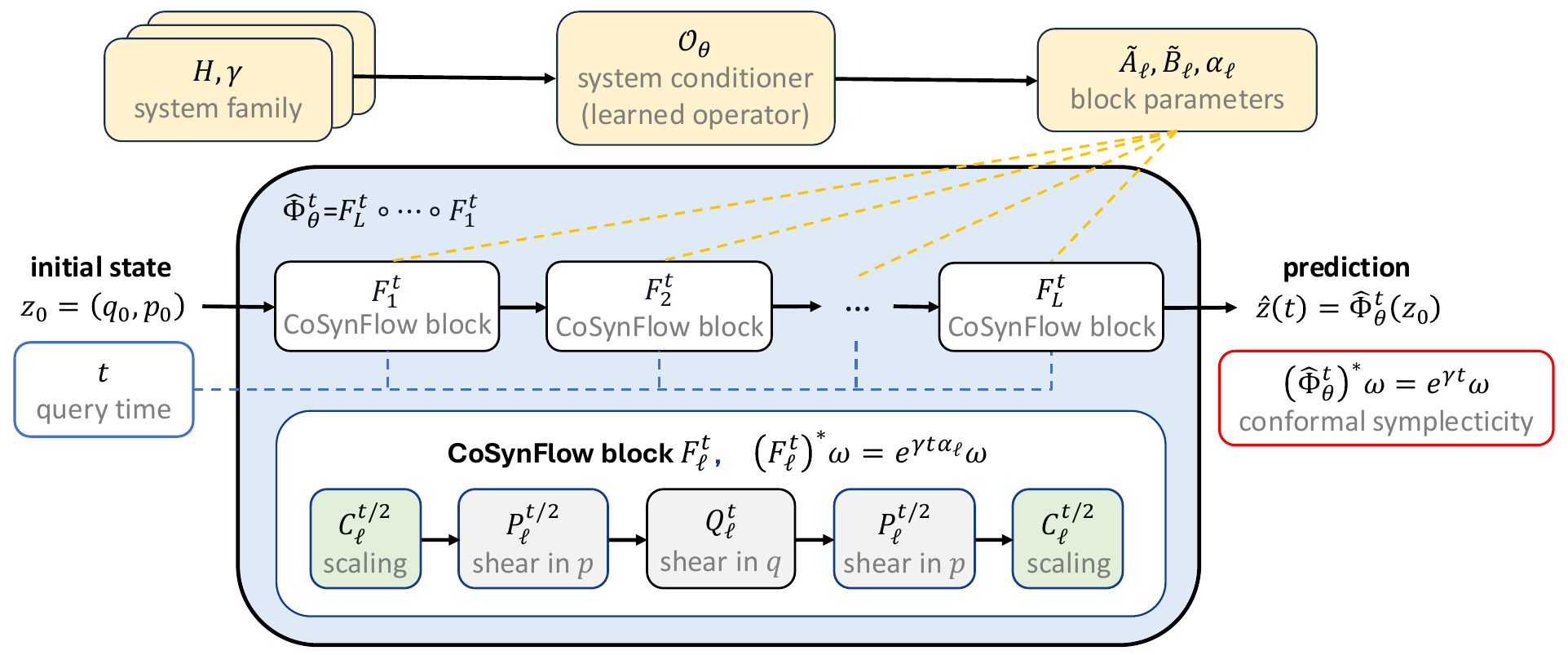}
    \caption{
Overview of CoSynFlow. The system conditioner $\mathcal{O}_{\theta}$
maps a system to the parameters of every flow block. The query time is
read by every block, so one trained model is evaluated at any $t$
instead of at a fixed step. The shears preserve $\omega$ and only the
scalings change it, so $\sum_{\ell}\alpha_\ell=1$ makes the map
conformal symplectic for arbitrary trainable parameters.
}
    \label{fig:cosynflow_architecture}
\end{figure*}

\section{Methodology: CoSynFlow}
\label{sec:methodology}

We first define the flow blocks for a fixed system and then condition
their trainable components on the system specification.
Figure~\ref{fig:cosynflow_architecture} summarizes the resulting model.

\subsection{Conformal Symplectic Flow Blocks}
\label{sec:flow_blocks}

For a fixed Hamiltonian $H$ and dissipation parameter $\gamma$, we
approximate the solution flow $\phi_t$ by a neural flow
$\widehat{\Phi}_{\theta}^{t}$ satisfying
\eqref{eq:true_flow_conformal_symplectic}. For each layer
$\ell=1,\ldots,L$, let $\widetilde{A}_{\ell}$ and
$\widetilde{B}_{\ell}$ be scalar potentials and define
\begin{align}
    P_{\ell}^{t/2}(q,p)
    &= \left( q,\; p-\tfrac{t}{2}\nabla_q\widetilde{A}_{\ell}(q;t) \right),
    \label{eq:p_shear}\\
    Q_{\ell}^{t}(q,p)
    &= \left( q+t\nabla_p\widetilde{B}_{\ell}(p;t),\; p \right),
    \label{eq:q_shear}\\
    C_{\ell}^{t/2}(q,p)
    &= \left( q,\; e^{\gamma t\alpha_\ell/2}p \right),
    \label{eq:conformal_scaling}
\end{align}
where $\sum_{\ell=1}^{L}\alpha_\ell=1$. The shears $P_{\ell}^{t/2}$ and
$Q_{\ell}^{t}$ are symplectic because they are generated by scalar
potentials, while
$(C_{\ell}^{t/2})^*\omega=e^{\gamma t\alpha_\ell/2}\omega$. We use the
symmetric block
\begin{align}
    F_{\ell}^{t}
    &= C_{\ell}^{t/2} \circ P_{\ell}^{t/2} \circ Q_{\ell}^{t}
       \circ P_{\ell}^{t/2} \circ C_{\ell}^{t/2},
    \label{eq:cosynflow_block}
\end{align}
which satisfies $(F_{\ell}^{t})^*\omega=e^{\gamma t\alpha_\ell}\omega$,
and the complete map
$\widehat{\Phi}_{\theta}^{t}=F_{L}^{t}\circ\cdots\circ F_{1}^{t}$
therefore satisfies
$(\widehat{\Phi}_{\theta}^{t})^*\omega=e^{\gamma t}\omega$. The explicit
factors of $t$ also give $\widehat{\Phi}_{\theta}^{0}=\mathrm{id}$ and
allow the model to be evaluated at continuously varying query times.

\subsection{System-Conditioned CoSynFlow}
\label{sec:system_conditioned_cosynflow}

The blocks above are defined for a single system. To share one model
across systems, we learn an operator that maps a system to every
system-dependent component of its blocks,
\begin{align*}
    \mathcal{O}_{\theta}:\
    (H,\gamma)
    \longmapsto
    \big(
        \{\alpha_\ell\}_{\ell=1}^{L},\;
        \{\widetilde{A}_{\ell},\widetilde{B}_{\ell}\}_{\ell=1}^{L}
    \big).
\end{align*}
Its trainable components read the system only through the descriptor
$h_H$ and the dissipation parameter, so they are conditioned on
$\xi=(h_H,\gamma)$. How this information reaches the potential networks
is not essential to the construction. Supplying it as an additional
input and using it to generate the network parameters both yield
potentials that vary with the system. The conditioned flow is
\begin{align}
    \widehat{\Phi}_{\theta}^{t}(\,\cdot\,;H,\xi)
    = F_{L,H,\xi}^{t}\circ\cdots\circ F_{1,H,\xi}^{t},
    \label{eq:system_conditioned_flow}
\end{align}
where $H$ enters through the block potentials and $\xi$ through the
trainable components. The condition specifies the system and the query
time selects which map of the flow is evaluated, so
$\mathcal{O}_{\theta}$ is evaluated once per system and reused across
query times.

For every fixed $(H,\xi,t)$, each shear potential still depends only on
$q$ or only on $p$, the scaling allocations remain independent of the
phase-space state, and $\sum_{\ell}\alpha_\ell(\xi)=1$ holds for every
$\xi$. The preceding geometric argument therefore applies to every
conditioned system.

\subsection{Model Instantiation}
\label{sec:model_instantiation}

Any component of the Hamiltonian that depends on the positions alone or
on the momenta alone can be evaluated directly in the corresponding
shear. We denote such components by $H_q$ and $H_p$, while the
descriptor enters only through $\mathcal{O}_{\theta}$. The scalar block
potentials are
\begin{align}
    \widetilde{A}_{\ell}(q;t,\xi)
    &= a_\ell(\xi)H_q(q) + \delta A_\ell(q;t,\xi),
    \label{eq:pot_A}\\
    \widetilde{B}_{\ell}(p;t,\xi)
    &= b_\ell(\xi)H_p(p) + \delta B_\ell(p;t,\xi),
    \label{eq:pot_B}
\end{align}
where $a_\ell$ and $b_\ell$ are scalar coefficients and $\delta A_\ell$
and $\delta B_\ell$ are trainable scalar potentials, one pair per block.
The first terms supply the dynamics that is known in closed form and the
trainable terms enrich the finite-depth flow. Since both enter only as
scalar potentials, the shears remain gradient maps, and their gradients
are taken by automatic differentiation.

A Fourier neural operator~\cite{fno} encodes $(h_H,\gamma)$ into a
system representation. A linear head maps it to $a_\ell$, $b_\ell$ and
the logits whose softmax gives $\alpha_\ell$, which enforces
$\sum_{\ell}\alpha_\ell(\xi)=1$. The same representation conditions
$\delta A_\ell$ and $\delta B_\ell$ through FiLM layers. Each network takes its own phase-space variable and features of the query time as input. Thus, although $\gamma$ also conditions the learned
components, the total conformal factor remains exactly $e^{\gamma t}$ by
construction.

The trainable potentials are initialized to zero and the coefficients to
$a_\ell=b_\ell=\alpha_\ell=1/L$, so that at initialization CoSynFlow is
exactly the classical conformal Strang splitting of $H_q+H_p$ with step
$t/L$. Training then adapts the coefficients and the potentials while
retaining exact conformal symplecticity.

\section{Theoretical Guarantees}
%
%


We give two guarantees for this construction. Theorem~\ref{thm:exact} states that the learned map is exactly conformal symplectic for every value of the trainable parameters. Theorem~\ref{thm:uat} states that this exact constraint does not limit the approximation capability of the flow. Full proofs are given in the appendix.

\begin{theorem}[Exact conformal symplecticity]
\label{thm:exact}
Fix a Hamiltonian $H$ and a system condition $\xi=(h_H,\gamma)$.
Suppose that for every $t\in[0,T]$ and every $\ell=1,\ldots,L$ the
scalar potentials $\widetilde{A}_{\ell}(\cdot\,;t,\xi)$ and
$\widetilde{B}_{\ell}(\cdot\,;t,\xi)$ are $C^{2}$ on $\mathbb{R}^{d}$,
and that the allocations $\alpha_\ell(\xi)$ are independent of the
phase-space variable and satisfy $\sum_{\ell=1}^{L}\alpha_\ell(\xi)=1$.
Let $\widehat{\Phi}_{\theta}^{t}(\cdot\,;H,\xi)$ be the conditioned
composition~\eqref{eq:system_conditioned_flow} of the symmetric
blocks~\eqref{eq:cosynflow_block}. Then, for arbitrary trainable
parameters and every $t\in[0,T]$, the map
$\widehat{\Phi}_{\theta}^{t}(\cdot\,;H,\xi)$ is a $C^{1}$
diffeomorphism of $\mathbb{R}^{2d}$ with an explicit inverse and
satisfies
\begin{align}
    \left(\widehat{\Phi}_{\theta}^{t}(\cdot\,;H,\xi)\right)^{*}\omega
    &= e^{\gamma t}\omega ,
    \label{eq:thm_exact}
\end{align}
equivalently
$D\widehat{\Phi}_{\theta}^{t}(z)^{\top}J\,
D\widehat{\Phi}_{\theta}^{t}(z)=e^{\gamma t}J$ for all
$z\in\mathbb{R}^{2d}$. Moreover
$\widehat{\Phi}_{\theta}^{0}(\cdot\,;H,\xi)=\mathrm{id}$ and
$\det D\widehat{\Phi}_{\theta}^{t}(z;H,\xi)=e^{d\gamma t}$.
\end{theorem}

The smoothness assumption holds for the model instantiation, since the
Hamiltonian components are smooth and the trainable potentials use
smooth activations. The system representation is a constant when
differentiating with respect to the phase-space variables, so each
shear remains the gradient of a scalar potential.

\begin{corollary}[Exact structure under composition]
\label{cor:composition}
Under the assumptions of Theorem~\ref{thm:exact}, for any query times
$t_1,\ldots,t_k\in[0,T]$,
\begin{align*}
    \left(\widehat{\Phi}_{\theta}^{t_k}\circ\cdots\circ
    \widehat{\Phi}_{\theta}^{t_1}\right)^{*}\omega
    = e^{\gamma(t_1+\cdots+t_k)}\omega ,
\end{align*}
where all maps are conditioned on the same $(H,\xi)$. The relation
therefore remains exact under any finite number of compositions, even
when $t_1+\cdots+t_k$ exceeds $T$ and trajectory errors accumulate.
\end{corollary}

\noindent\emph{Proof sketch.}
Each gradient shear has a block-triangular Jacobian whose off-diagonal
block is a symmetric Hessian, so every shear is symplectic, and the two
half-scalings in block $\ell$ contribute the factor
$e^{\gamma t\alpha_\ell(\xi)}$. Conformal factors multiply under
composition, so $\sum_{\ell}\alpha_\ell(\xi)=1$ gives the total factor
$e^{\gamma t}$. Every factor is invertible by reversing its own update,
every factor is the identity at $t=0$ because of the explicit factors of
$t$, and the shears have determinant one while the scalings give
$e^{d\gamma t}$. Details are in the appendix.

We next show that this exact constraint costs no approximation
capability. The result is stated for a single system, and $\xi$ is
suppressed from the notation.

\begin{theorem}[Uniform approximation]
\label{thm:uat}
Let $K\subset\mathbb{R}^{2d}$ be compact, let $T>0$ and $\gamma<0$,
let $H\in C^{2}(\mathbb{R}^{2d})$, and let $\varphi^{t}$ be the flow
of the conformal Hamiltonian system
\eqref{eq:dissipative_hamiltonian}. Assume that there exists a
compact set $K'\subset\mathbb{R}^{2d}$ such that
\begin{align*}
    \varphi^{s}(K)
    \subset
    \operatorname{int}K'
\end{align*}
for every $s\in[0,T]$. Then, for every $\varepsilon>0$, there exist a
depth $L$ and parameters $\theta_L$, independent of $t$, such that
\begin{align*}
    \sup_{z\in K}
    \sup_{t\in[0,T]}
    \left\|
        \widehat{\Phi}_{\theta_L}^{t}(z)
        -
        \varphi^{t}(z)
    \right\|
    <
    \varepsilon.
\end{align*}
Moreover, for every $t\in[0,T]$,
\begin{align*}
    \left(
        \widehat{\Phi}_{\theta_L}^{t}
    \right)^*\omega
    =
    e^{\gamma t}\omega.
\end{align*}
\end{theorem}

\noindent\emph{Proof sketch.}
The change of variables $y=(q,e^{-\gamma t}p)$ turns the conformal flow
into the symplectic flow of the time-dependent Hamiltonian
$\widetilde{H}(t,y)=e^{-\gamma t}H(y_q,e^{\gamma t}y_p)$, which is again
$C^{2}$ in $y$. On a compact set $\widetilde{H}$ is approximated by a
Hamiltonian that is polynomial in the phase-space variable, and the
resulting flow is split into substeps, each of which is approximated by
the flow of a Hamiltonian whose two parts depend on the positions and on
the momenta separately~\cite{turaev}. Such flows are realized by the
gradient shears, whose potentials take the query time as an input, and
the scalings are moved between the blocks using
$C_\lambda\circ P_A=P_{\lambda A}\circ C_\lambda$ and
$C_\lambda\circ Q_B=Q_{B_\lambda}\circ C_\lambda$ with
$B_\lambda(p)=\lambda B(p/\lambda)$. A single parameter vector therefore
covers the whole interval, and Theorem~\ref{thm:exact} supplies the
exact conformal factor. The argument adapts the universality result for
symplectic neural flows~\cite{snf} to the conformal setting, the
additional ingredients being the change of variables and the exact
reinsertion of the scalings into the block structure. Details are in the
appendix.


\section{Numerical Experiments}
\label{sec:experiments}

\subsection{Experimental Setup}
\label{sec:experimental_setup}

\paragraph{Data.}
We instantiate the construction on two-degree-of-freedom natural
Hamiltonians $H(q,p)=\tfrac12\lVert p\rVert^{2}+V(q)$ on the position
domain $\mathcal{Q}=[-3,3]^{2}$, so that $H_q=V$ and
$H_p=\tfrac12\lVert p\rVert^{2}$ in \eqref{eq:pot_A} and
\eqref{eq:pot_B}. Each potential is a random quartic backbone plus a
Gaussian process perturbation of amplitude $\delta\in[0.2,1.5]$, which
gives single well, double well and multiple well landscapes, and the
dissipation parameter is drawn log-uniformly with
$\lvert\gamma\rvert\in[0.1,0.4]$. The descriptor consists of the values
of $V$ on a uniform $32\times32$ grid over $\mathcal{Q}$. Reference
trajectories are integrated in double precision on $[0,T]$ with $T=4$,
from initial conditions confined to the domain by an energy budget. The
dataset contains $38{,}000$ training, $1{,}000$ validation and $1{,}000$
test systems with $32$ trajectories each, and training pairs
$(z,\tau,z^{+})$ are sampled online with stratified time gaps. Sampling
ranges and integration settings are given in the appendix.

\paragraph{Evaluation systems.}
Beyond the held-out random systems we evaluate on four analytic
benchmark systems not seen during training. These are an anisotropic
oscillator (B1), a coupled Duffing system (B2), a Mexican hat potential
(B3) and a quartic coupled oscillator (B4). They share the quartic form
of the backbone and carry no Gaussian process perturbation, so they lie
outside the training distribution, in which every system has
$\delta\geq0.2$. Each is evaluated from $64$ initial conditions drawn
from its own energy sublevel set. Their coefficients are given in the
appendix.

\paragraph{Metrics.}
On the held-out random systems we report the normalized mean squared training loss $\mathcal{L}_{\mathrm{pred}}$. On the benchmark systems the state
error $\mathcal{E}_{z}(t)$ is
$\lVert\widehat z_i(t)-z_i(t)\rVert/\sigma$ and the energy error
$\mathcal{E}_{H}(t)$ is
$\lvert H(\widehat z_i(t))-H(z_i(t))\rvert/(E_{\max}-V_{\min})$, both
averaged over $N=64$ initial conditions, with
$\sigma=\sqrt{E_{\max}-V_{\min}}$. Geometric consistency is measured by
the conformal symplectic structure error
\begin{align*}
    \mathcal{E}_\omega(z,t)
    &= \frac{\big\lVert
        D\widehat{\Phi}_{\theta}^{t}(z)^{\top}J\,
        D\widehat{\Phi}_{\theta}^{t}(z)-e^{\gamma t}J
    \big\rVert_{F}}
    {\big\lVert e^{\gamma t}J\big\rVert_{F}},
\end{align*}
evaluated in double precision.

\paragraph{Long-horizon protocol.}
Predictions beyond the trained horizon are produced by composition
rather than by a single large query. We partition $[0,10T]$ into ten
windows of length $T$, advance the anchors by
$z_{k+1}=\widehat{\Phi}_{\theta}^{T}(z_k)$, and fill each window with
direct queries $\widehat{\Phi}_{\theta}^{\tau}(z_k)$ for $\tau\in(0,T]$.
Every query time therefore lies inside the trained range and crossing a
window boundary is pure composition, so
Corollary~\ref{cor:composition} applies at every step. These results use
$\lvert\gamma\rvert=0.1$, the least damped end of the training range,
since at larger dissipation the state has decayed to the fixed point
well before $t=10T$.

\paragraph{Models and training.}
We compare CoSynFlow, which uses $L=12$ blocks, with five models.
\textbf{Symplectic Flow} uses the same conditioned blocks with the
conformal scaling removed, so it represents ordinary symplectic flow
models and is also the direct ablation of the conformal component.
\textbf{MLP Flow} is an unstructured direct-flow model with the same
conditioner, and \textbf{MLP Flow with soft constraint} adds a conformal
symplectic Jacobian penalty to it. \textbf{DeepONet} reads the system in
its branch network and $(z_0,t)$ in its trunk network, and
\textbf{Transolver Flow} replaces the conditioner of MLP Flow with a
Physics-Attention encoder. We further consider two ablations of
CoSynFlow. \textbf{Fixed Allocation} replaces the learned scaling
allocations by $\alpha_\ell=1/L$, and \textbf{No Correction} removes the
trainable potentials. 
All models share the same training data, normalization and optimization
budget, and every baseline matches the parameter count of CoSynFlow
within $3\%$, with DeepONet and Transolver using their native encoders.
Results are means over three seeds, each run on a single H100 GPU
partition. Remaining settings are in the appendix.


\subsection{Accuracy and Geometric Consistency}
\label{sec:main_results}

\begin{table}[t]
\centering
\caption{
Accuracy on held-out random systems and geometric consistency of the
learned map at $t=2$. Lower is better in both columns.
}
\label{tab:main}

\setlength{\tabcolsep}{4pt}
\small
\begin{tabular}{lcc}
\toprule
Model & In-distribution MSE & Structure error \\
\midrule
\multicolumn{3}{l}{\textit{Baselines}} \\
Symplectic Flow & $1.3_{\pm0.0}{\times}10^{-2}$ & $8.2_{\pm0.0}{\times}10^{-1}$ \\
MLP Flow & $1.2_{\pm0.0}{\times}10^{-2}$ & $2.6_{\pm0.3}{\times}10^{-1}$ \\
MLP Flow (soft) & $1.9_{\pm0.0}{\times}10^{-2}$ & $1.3_{\pm0.2}{\times}10^{-1}$ \\
DeepONet & $2.7_{\pm0.0}{\times}10^{-2}$ & $7.4_{\pm1.3}{\times}10^{-1}$ \\
Transolver & $7.8_{\pm0.2}{\times}10^{-3}$ & $2.2_{\pm0.3}{\times}10^{-1}$ \\
\midrule
\multicolumn{3}{l}{\textit{Ablations}} \\
\quad Fixed Allocation & $3.9_{\pm0.9}{\times}10^{-4}$ & $1.6_{\pm0.2}{\times}10^{-15}$ \\
\quad No Correction & $6.1_{\pm0.9}{\times}10^{-4}$ & $9.9_{\pm2.7}{\times}10^{-16}$ \\
\midrule
\textbf{CoSynFlow} & $\bm{3.4_{\pm0.2}{\times}10^{-4}}$ & $1.3_{\pm0.2}{\times}10^{-15}$ \\
\bottomrule
\end{tabular}

\end{table}

In Table~\ref{tab:main}, CoSynFlow attains the lowest in-distribution MSE, with a margin over the best baseline far larger than the seed variation, and both ablations outperform every baseline by more than an order of magnitude. The
structure error separates the models far more sharply. CoSynFlow and the
two ablations keep it at double precision round-off, fourteen orders of
magnitude below the unstructured and softly constrained models, as
predicted by Theorem~\ref{thm:exact}. The differences among the three
are round-off noise. After sixteen compositions, a total time of $7.5T$,
the three are still of order $10^{-13}$ while every other model exceeds
$20$, which is the quantitative form of Corollary~\ref{cor:composition}.

Fixing the allocations to $\alpha_\ell=1/L$ preserves the exact
structure, since only the normalization enters the guarantee, and costs
$15\%$ in accuracy. Removing the trainable potentials leaves the
classical conformal Strang splitting and costs a factor of $1.8$.
Removing the conformal scaling instead gives Symplectic Flow, whose
structure error is exactly $e^{-\gamma t}-1$ and is identical across
seeds, since it is fixed by the geometry the model enforces rather than
by what it learns.
The three ablations therefore separate the roles of the conformal
factor, the allocation and the learned potentials.

\subsection{Long-Horizon Prediction}
\label{sec:long_horizon}

\begin{table*}[t]
\centering
\caption{
Long-horizon error on the four benchmark systems, averaged over
$t\in(0,10T]$ and over $64$ initial conditions at
$\lvert\gamma\rvert=0.1$. Predictions beyond $T$ are obtained by
composing the learned flow over ten windows. The last column is the
energy error averaged over the four systems. DeepONet diverges
on B3. Per-system energy errors are in the appendix.
}
\label{tab:long}
\setlength{\tabcolsep}{4pt}
\small
\begin{tabular}{lccccc|c}
\toprule
& \multicolumn{5}{c|}{State error $\mathcal{E}_z$} & Energy error $\mathcal{E}_H$ \\
\cmidrule(lr){2-6}\cmidrule(lr){7-7}
Model & B1 & B2 & B3 & B4 & Mean & Mean \\
\midrule
\multicolumn{7}{l}{\textit{Baselines}} \\
Symplectic Flow & $8.2_{\pm0.0}{\times}10^{-1}$ & $1.4_{\pm0.0}{\times}10^{0}$ & $1.3_{\pm0.0}{\times}10^{0}$ & $7.8_{\pm0.0}{\times}10^{-1}$ & $1.1_{\pm0.0}{\times}10^{0}$ & $5.1_{\pm0.0}{\times}10^{-1}$ \\
MLP Flow & $1.6_{\pm0.4}{\times}10^{-1}$ & $4.2_{\pm0.4}{\times}10^{-1}$ & $7.0_{\pm0.3}{\times}10^{-1}$ & $1.5_{\pm0.1}{\times}10^{-1}$ & $3.6_{\pm0.2}{\times}10^{-1}$ & $2.7_{\pm0.1}{\times}10^{-2}$ \\
MLP Flow (soft) & $2.0_{\pm0.3}{\times}10^{-1}$ & $6.3_{\pm0.6}{\times}10^{-1}$ & $8.0_{\pm0.5}{\times}10^{-1}$ & $2.7_{\pm0.1}{\times}10^{-1}$ & $4.7_{\pm0.2}{\times}10^{-1}$ & $2.6_{\pm0.2}{\times}10^{-2}$ \\
DeepONet & $2.2_{\pm0.6}{\times}10^{-1}$ & $6.8_{\pm0.2}{\times}10^{-1}$ & $1.4_{\pm0.8}{\times}10^{0}$ & $2.8_{\pm0.0}{\times}10^{-1}$ & $6.4_{\pm2.2}{\times}10^{-1}$ & $1.2_{\pm2.0}{\times}10^{6}$ \\
Transolver & $9.9_{\pm1.1}{\times}10^{-2}$ & $3.7_{\pm0.3}{\times}10^{-1}$ & $5.9_{\pm0.2}{\times}10^{-1}$ & $1.2_{\pm0.1}{\times}10^{-1}$ & $2.9_{\pm0.1}{\times}10^{-1}$ & $2.2_{\pm0.2}{\times}10^{-2}$ \\
\midrule
\multicolumn{7}{l}{\textit{Ablations}} \\
\quad Fixed Allocation & $\bm{1.4_{\pm0.4}{\times}10^{-2}}$ & $5.2_{\pm1.4}{\times}10^{-2}$ & $1.8_{\pm0.1}{\times}10^{-1}$ & $2.6_{\pm0.7}{\times}10^{-2}$ & $6.8_{\pm0.7}{\times}10^{-2}$ & $2.2_{\pm0.4}{\times}10^{-3}$ \\
\quad No Correction & $5.0_{\pm0.1}{\times}10^{-2}$ & $6.2_{\pm0.4}{\times}10^{-2}$ & $1.8_{\pm0.1}{\times}10^{-1}$ & $4.7_{\pm0.2}{\times}10^{-2}$ & $8.5_{\pm0.4}{\times}10^{-2}$ & $2.7_{\pm0.5}{\times}10^{-3}$ \\
\midrule
\textbf{CoSynFlow} & $1.5_{\pm0.2}{\times}10^{-2}$ & $\bm{4.3_{\pm0.3}{\times}10^{-2}}$ & $\bm{1.6_{\pm0.1}{\times}10^{-1}}$ & $\bm{2.4_{\pm0.6}{\times}10^{-2}}$ & $\bm{6.0_{\pm0.4}{\times}10^{-2}}$ & $\bm{1.8_{\pm0.1}{\times}10^{-3}}$ \\
\bottomrule
\end{tabular}

\end{table*}

\begin{figure}[t]
    \centering
    \includegraphics[width=\linewidth]{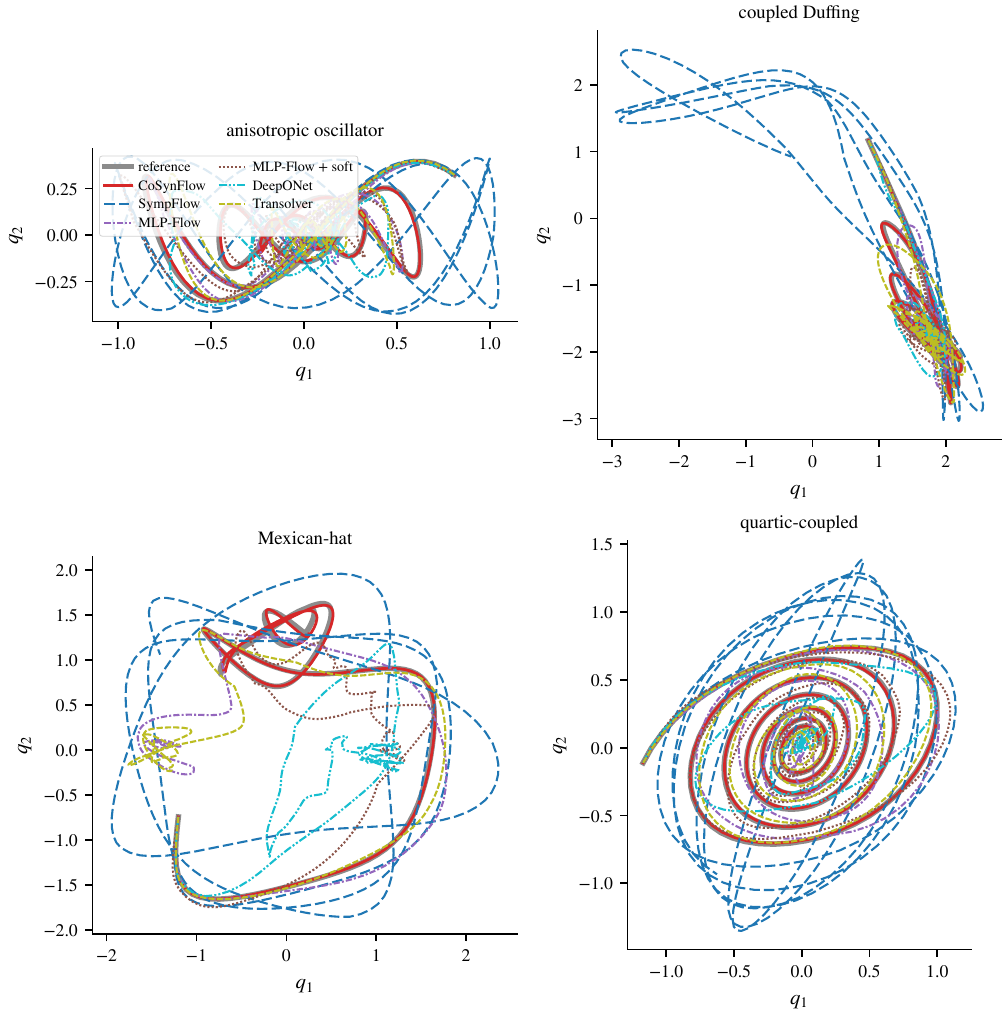}
    \caption{
    Trajectories in the position plane $(q_1,q_2)$ for the four benchmark
systems over $0\leq t\leq 10T$ at $\lvert\gamma\rvert=0.1$, one initial
condition each, with the reference in thick gray. Predictions are
obtained by composing the learned flow over ten windows of length $T$.
    }
    \label{fig:phase}
\end{figure}

\begin{figure}[t]
    \centering
    \includegraphics[width=\linewidth]{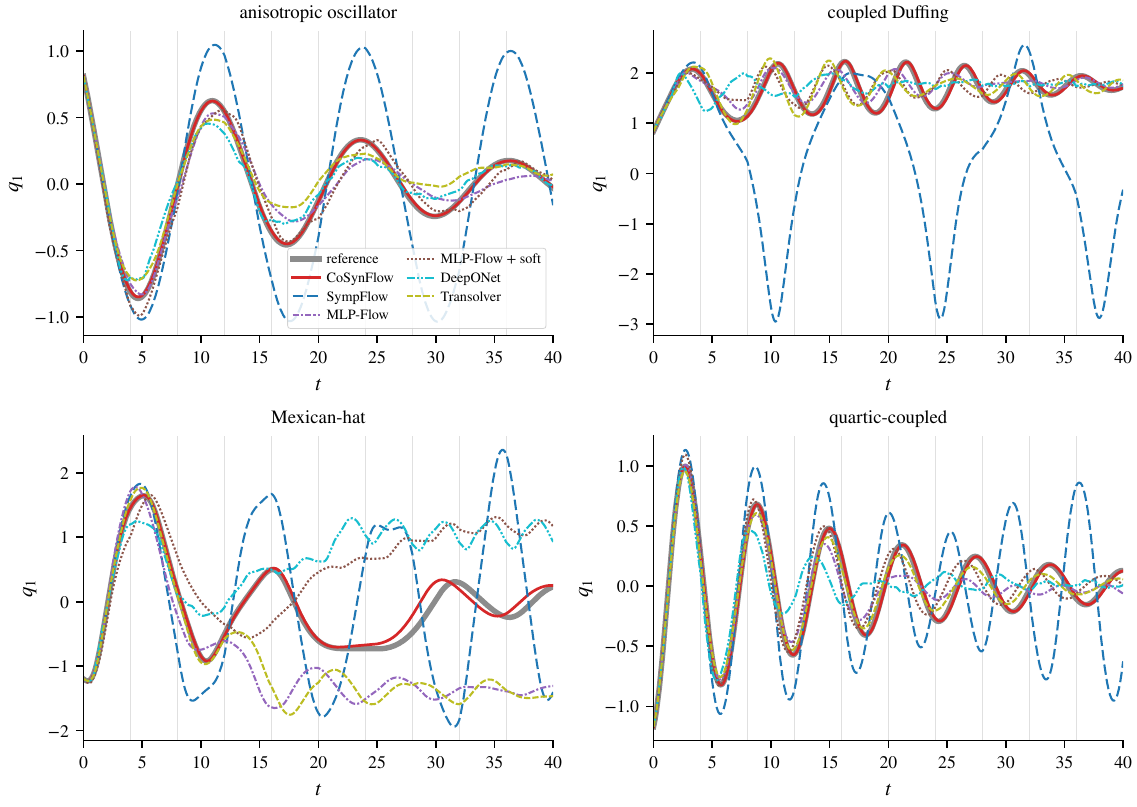}
    \caption{
    First position component $q_1$ versus time for the same trajectories as
Figure~\ref{fig:phase}, with the reference in thick gray. Faint vertical
lines mark the ten composition window boundaries, inside which the model
is queried directly at times in $(0,T]$.
    }
    \label{fig:q1}
\end{figure}

Table~\ref{tab:long} reports accuracy over ten times the trained
horizon. CoSynFlow attains the lowest state error on three of the four benchmark systems and the lowest average error in both metrics, ahead of the best baseline by nearly a factor of five in state error and twelve in energy error. 
Symplectic Flow has the largest state error of all
models, three times that of the unconstrained MLP Flow, so ordinary
symplecticity is not a weaker constraint here but the wrong one, and
enforcing it exactly is worse than enforcing nothing at all.

Figures~\ref{fig:phase} and~\ref{fig:q1} show why the gap widens with
time. The reference trajectories spiral toward a potential minimum.
CoSynFlow contracts phase-space volume at exactly the rate $e^{d\gamma t}$ and tracks the reference on all four systems over the full horizon. Symplectic Flow preserves volume and cannot contract at
all, so its amplitude never decays, and on the coupled Duffing system it
leaves the well in which the true solution settles. The unstructured
models contract at a learned rather than prescribed rate, lose the
oscillation phase within a few windows, and on the Mexican-hat potential
settle in a different well from the reference, so their error there is
qualitative rather than numerical, whereas the residual deviation of CoSynFlow on that system stays within the correct well. Theorem~\ref{thm:exact} constrains how the map contracts phase space, not where a trajectory goes, so the amplitude and the phase are both learned. The predictions also cross the ten window boundaries
without a visible break, so composition introduces no artefact of its
own. 

When a single system is learned in isolation rather than across the
family, the same architecture reaches an MSE as low as $3\times10^{-7}$,
so the errors reported here reflect the cost of amortizing one model
over a family rather than a limitation of the construction.

\subsection{Extrapolation in Dissipation}
\label{sec:extrapolation}

\begin{figure}[!tb]
    \centering
    \includegraphics[width=\linewidth]{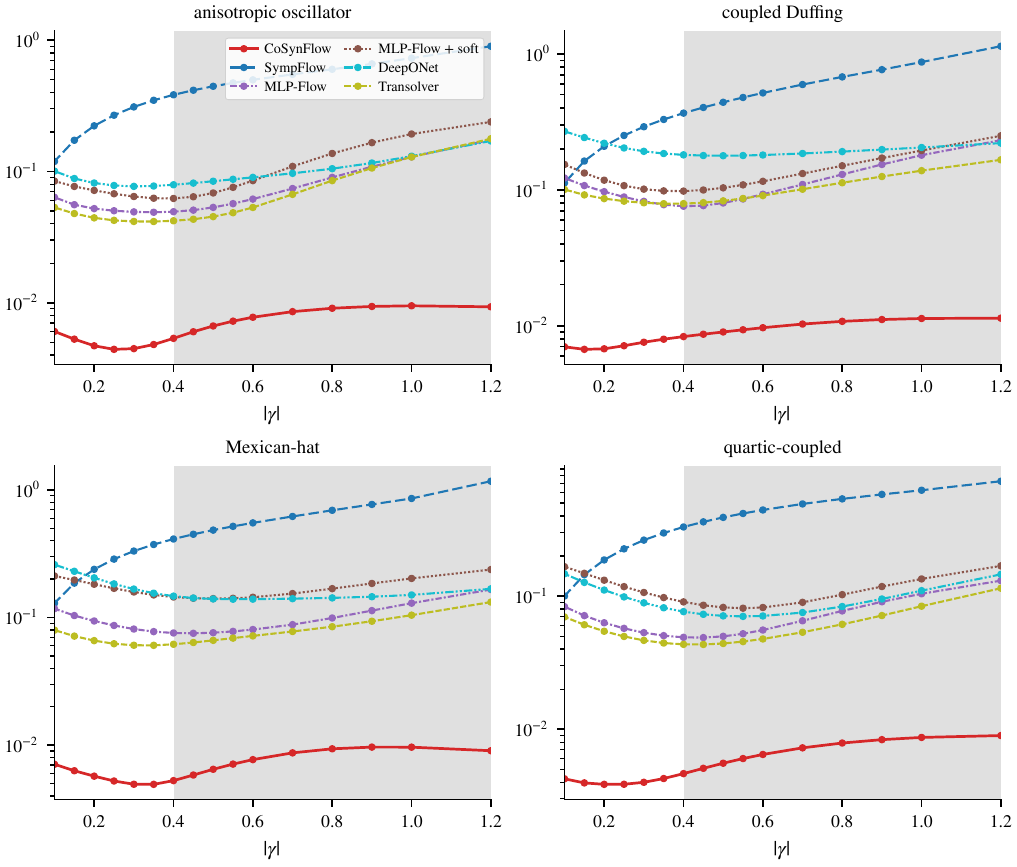}
    \caption{
    State error $\mathcal{E}_z$ averaged over $t\in(0,T]$, on a logarithmic scale. The shaded region lies outside the
training range $\lvert\gamma\rvert\in[0.1,0.4]$ and extends to three
times its upper end.
    }
    \label{fig:gamma}
\end{figure}

Figure~\ref{fig:gamma} varies the dissipation parameter over an interval
that reaches three times the largest value seen during training. Every
model has a shallow minimum inside the training range and degrades in
both directions, but the error of CoSynFlow varies by less than a factor of three over
the whole sweep and stays below every baseline at every value of
$\lvert\gamma\rvert$. The margin grows from about ten
times at the edge of the training range to about twenty times at the far
end. The scaling factor $e^{\gamma t\alpha_\ell}$ is applied
analytically, so the conformal factor is correct at any $\gamma$.
Accuracy outside the training range remains empirical, since the
coefficients and potentials are also conditioned on $\gamma$, which the
baselines must learn in full. Symplectic Flow degrades fastest and
monotonically, because the factor it fails to represent grows with
$\lvert\gamma\rvert$.

\subsection{Physics-Informed Training}
\label{sec:physics_informed}

Because $\widehat{\Phi}_{\theta}^{t}$ is differentiable in the query
time, the residual of the governing equation can be evaluated directly
on the learned flow,
\begin{align*}
    \mathcal{L}_{\mathrm{PI}}(\theta)
    = \frac{1}{\lvert\mathcal{C}\rvert}
    \sum_{(z,t)\in\mathcal{C}}
    \frac{\big\lVert
        \partial_t\widehat{\Phi}_{\theta}^{t}(z)
        -X\big(\widehat{\Phi}_{\theta}^{t}(z)\big)
    \big\rVert^{2}}
    {4\,(\sigma/T)^{2}},
\end{align*}
where $X$ is the vector field \eqref{eq:dissipative_hamiltonian} and
$\sigma/T$ is a velocity scale that makes the term dimensionless. The
collocation set $\mathcal{C}$ consists of states drawn from the energy
sublevel set and query times drawn from $[0,T]$, resampled at every
step, so the term supplies supervision at no data cost. A model of a
fixed-step map cannot form this residual, since its output is defined
only at multiples of the step and carries no derivative in the query
time.

We test the term on a single fixed system, the coupled Duffing
benchmark B2 at $\lvert\gamma\rvert=0.2$, with only two reference
trajectories, the regime in which supervision that needs no data is
expected to matter. The same architecture is trained twice under an
identical budget, once with $\mathcal{L}_{\mathrm{pred}}$ alone and once
with $\mathcal{L}_{\mathrm{pred}}+\lambda\mathcal{L}_{\mathrm{PI}}$ at
$\lambda=1$, with three seeds each, and errors are measured on $64$
held-out trajectories.

\begin{table}[t]
\centering
\caption{
Physics-informed training of CoSynFlow on a single system with two
training trajectories. The two rows differ only in the training
objective.
}
\label{tab:physics_informed}
\setlength{\tabcolsep}{5pt}
\small
\begin{tabular}{lcc}
\toprule
Training objective & State error & Energy error \\
\midrule
$\mathcal{L}_{\mathrm{pred}}$ & $1.2_{\pm0.5}{\times}10^{-2}$ & $2.9_{\pm1.7}{\times}10^{-3}$ \\
$\mathcal{L}_{\mathrm{pred}}+\lambda\mathcal{L}_{\mathrm{PI}}$ & $\bm{2.9_{\pm0.2}{\times}10^{-3}}$ & $\bm{4.9_{\pm0.4}{\times}10^{-4}}$ \\
\bottomrule
\end{tabular}

\end{table}

In Table~\ref{tab:physics_informed}, the residual term lowers the state
error by a factor of four and the energy error by six, and a similar
gain is observed after composition to $10T$. Every seed with the term
beats every seed without it, and the spread across seeds falls by almost
an order of magnitude. The structure error is unaffected by the additional term, as reported in the appendix. The term therefore changes what the model learns and never whether the geometry holds.

\section{Conclusion}

We presented CoSynFlow, a system-conditioned neural flow for the
solution maps of dissipative Hamiltonian dynamics. Each time-$t$ map
composes symmetric blocks of gradient shears and conformal scalings. The
shears are symplectic because they are generated by scalar potentials,
and the scaling factors are normalized to multiply to $e^{\gamma t}$, so
the map is exactly conformal symplectic for arbitrary trainable
parameters, at every query time and under any number of compositions.
Because the query time is a continuous differentiable input, one trained
model is evaluated at any time, composed beyond the training horizon,
and trained with a physics-informed residual. Conditioning on a
Hamiltonian descriptor and the dissipation parameter extends the
construction to a family of systems without weakening the guarantee.

Across four unseen benchmark systems the model attains the lowest state
and energy error among structured and unstructured baselines at a
matched parameter count, while holding the structure error at double
precision round-off, both at a single query time and under repeated
composition. Accuracy is retained well outside the range of dissipation rates seen during training, where the conformal factor is applied analytically. The comparison also shows that a structure can be enforced exactly and still be the wrong one, since ordinary symplecticity performs worse here than no constraint at all.

Being differentiable and exactly structured, the learned map can stand
in for a solver in loops where many solutions of one system family are
required.
The
physics-informed residual needs no reference trajectory, which points to
building such surrogates from far less data. More broadly, factoring a
flow into structure preserving parts and one analytic part that carries
the prescribed geometric weight is not specific to conformal
symplecticity, and we expect it to transfer to other structured
dynamics.

\appendix

\section{Proofs}

\subsection{Proof of Theorem 1 and Corollary 1}

\begin{proof}[Proof of Theorem~1]
Write $\omega_z(u,v)=u^{\top}Jv$, so that for a $C^{1}$ map $\Psi$ we
have
\begin{align*}
    (\Psi^{*}\omega)_z(u,v)=u^{\top}D\Psi(z)^{\top}J\,D\Psi(z)\,v,
\end{align*}
and
$\Psi^{*}\omega=c\,\omega$ holds precisely when
$D\Psi(z)^{\top}J\,D\Psi(z)=cJ$ for every $z$. We call such a map
conformal symplectic with factor $c$.
If $\Psi_1$ and $\Psi_2$ are conformal symplectic with factors $c_1$ and
$c_2$, the chain rule gives
\begin{equation}
    D(\Psi_2\circ\Psi_1)^{\top}J\,D(\Psi_2\circ\Psi_1)
    = D\Psi_1^{\top}\big[D\Psi_2^{\top}J\,D\Psi_2\big]D\Psi_1
    = c_1c_2\,J ,
    \label{eq:app_mult}
\end{equation}
where the bracket is evaluated at $\Psi_1(z)$. Determinants also
multiply, since $\det(MN)=\det M\det N$.

Set
$A_{\ell}:=\tfrac{t}{2}\At_{\ell}(\,\cdot\,;t,\xi)$,
$B_{\ell}:=t\,\Bt_{\ell}(\,\cdot\,;t,\xi)$ and
$\lambda_{\ell}:=e^{\gamma t\alpha_\ell(\xi)/2}$.
By hypothesis $A_{\ell}$ and $B_{\ell}$ are $C^{2}$ on $\R^{d}$, and $\lambda_{\ell}>0$. 
Because
$\alpha_\ell(\xi)$ does not depend on the phase-space variable, each
$\lambda_{\ell}$ is a constant, and because the system representation is
held fixed while differentiating in $(q,p)$, the factors of
$F_{\ell}^{t}$ are the three maps
\begin{equation*}
    P_{A}(q,p)=\big(q,\;p-\nabla A(q)\big),
    \qquad
    Q_{B}(q,p)=\big(q+\nabla B(p),\;p\big),
    \qquad
    C_{\lambda}(q,p)=(q,\;\lambda p)
\end{equation*}
with $A=A_{\ell}$, $B=B_{\ell}$ and $\lambda=\lambda_{\ell}$.

For the shears, write $S:=\nabla^{2}A(q)$ and $R:=\nabla^{2}B(p)$. Then
\begin{equation*}
    DP_{A}^{\top}J\,DP_{A}
    = \begin{pmatrix} I & -S^{\top}\\ 0 & I\end{pmatrix}
      \begin{pmatrix} -S & I\\ -I & 0\end{pmatrix}
    = \begin{pmatrix} S^{\top}-S & I\\ -I & 0\end{pmatrix}
    = J ,
    \qquad
    DQ_{B}^{\top}J\,DQ_{B}
    = \begin{pmatrix} 0 & I\\ -I & R^{\top}-R\end{pmatrix}
    = J ,
\end{equation*}
since the Hessian of a $C^{2}$ scalar function is symmetric. Both
Jacobians are block unitriangular, so
$\det DP_{A}=\det DQ_{B}=1$. For the scaling,
$DC_{\lambda}=\operatorname{diag}(I,\lambda I)$ gives
\begin{equation*}
    DC_{\lambda}^{\top}J\,DC_{\lambda}
    = \begin{pmatrix} 0 & \lambda I\\ -\lambda I & 0\end{pmatrix}
    = \lambda J ,
    \qquad
    \det DC_{\lambda}=\lambda^{d} .
\end{equation*}

Each block $F_{\ell}^{t}$ consists of three shears of factor one and two
scalings of factor $\lambda_{\ell}$, so by \eqref{eq:app_mult}
\begin{equation*}
    \big(F_{\ell}^{t}\big)^{*}\omega
    = \lambda_{\ell}^{2}\,\omega
    = e^{\gamma t\alpha_\ell(\xi)}\,\omega ,
    \qquad
    \det DF_{\ell}^{t}=e^{d\gamma t\alpha_\ell(\xi)} ,
\end{equation*}
and composing the $L$ blocks with
$\sum_{\ell=1}^{L}\alpha_\ell(\xi)=1$ yields
\begin{equation*}
    \big(\Phih^{t}\big)^{*}\omega
    = \prod_{\ell=1}^{L}e^{\gamma t\alpha_\ell(\xi)}\,\omega
    = e^{\gamma t}\omega ,
    \qquad
    \det D\Phih^{t}
    = \prod_{\ell=1}^{L}e^{d\gamma t\alpha_\ell(\xi)}
    = e^{d\gamma t}.
\end{equation*}

Each factor is a bijection of $\R^{2d}$ whose
inverse is a map of the same type,
\begin{equation*}
    P_{A}^{-1}(q,p)=\big(q,\;p+\nabla A(q)\big),
    \qquad
    Q_{B}^{-1}(q,p)=\big(q-\nabla B(p),\;p\big),
    \qquad
    C_{\lambda}^{-1}=C_{1/\lambda} ,
\end{equation*}
well defined because $P_{A}$ leaves $q$ unchanged, $Q_{B}$ leaves $p$
unchanged and $\lambda>0$. All of these are $C^{1}$ since
$A,B\in C^{2}$, so every factor is a global $C^{1}$ diffeomorphism of
$\R^{2d}$ and so is $\Phih^{t}$, with
\begin{equation*}
    \big(\Phih^{t}\big)^{-1}
    = \big(F_{1}^{t}\big)^{-1}\circ\cdots\circ\big(F_{L}^{t}\big)^{-1},
    \qquad
    \big(F_{\ell}^{t}\big)^{-1}
    = C_{\lambda_{\ell}}^{-1}\circ P_{A_{\ell}}^{-1}
      \circ Q_{B_{\ell}}^{-1}\circ P_{A_{\ell}}^{-1}
      \circ C_{\lambda_{\ell}}^{-1} .
\end{equation*}
Finally, the prefactors in $A_{\ell}$ and $B_{\ell}$ do not depend on
the phase-space variable, so at $t=0$ both gradients vanish and
$\lambda_{\ell}=1$. Every factor is then the identity and
$\Phih^{0}=\mathrm{id}$.
\end{proof}

\begin{proof}[Proof of Corollary~1]
The hypotheses of Theorem~1 hold for every $t\in[0,T]$, so
$\big(\Phih^{t_i}\big)^{*}\omega=e^{\gamma t_i}\omega$ for each $i$, and
applying \eqref{eq:app_mult} repeatedly gives
\begin{equation*}
    \Big(\Phih^{t_k}\circ\cdots\circ\Phih^{t_1}\Big)^{*}\omega
    = \prod_{i=1}^{k}e^{\gamma t_i}\,\omega
    = e^{\gamma(t_1+\cdots+t_k)}\omega .
\end{equation*}
The identity holds for any $t_1,\ldots,t_k \in [0,T]$, including sums that exceed $T$.
\end{proof}

\subsection{Proof of Theorem 2}

\begin{proof}[Proof of Theorem~2]
\emph{Change of variables.}
For $s\in\R$ let $\Lambda_{s}:=C_{e^{-\gamma s}}$, that is
$\Lambda_{s}(q,p)=(q,e^{-\gamma s}p)$, so $\Lambda_{0}=\mathrm{id}$. Let $z(t)=\varphi^{t}(z_0)$ and put $y(t):=\Lambda_{t}(z(t))$, whose
components are written $y=(y_q,y_p)$, so that $y_q=q$ and
$y_p=e^{-\gamma t}p$. Then
\begin{equation*}
    \dot y_q = \nabla_pH(z),
    \qquad
    \dot y_p = -\gamma e^{-\gamma t}z_p + e^{-\gamma t}\dot z_p
    = -e^{-\gamma t}\nabla_qH(z) ,
\end{equation*}
and with
\begin{equation}
    \Ht(t,y) := e^{-\gamma t}H\big(y_q,\;e^{\gamma t}y_p\big)
    \label{eq:app_Htilde}
\end{equation}
one has $\nabla_{y_q}\Ht(t,y)=e^{-\gamma t}\nabla_qH(z)$ and
$\nabla_{y_p}\Ht(t,y)=\nabla_pH(z)$ for $z=(y_q,e^{\gamma t}y_p)$, so
$\dot y=J\nabla_y\Ht(t,y)$. Writing $\psi^{t}$ for the flow of this
equation and using 
$\Lambda_{0}=\mathrm{id}$,
\begin{equation}
    \varphi^{t} = C_{e^{\gamma t}}\circ\psi^{t} .
    \label{eq:app_conjugation}
\end{equation}
By \eqref{eq:app_Htilde} and $H\in C^{2}(\R^{2d})$, the function
$\Ht(t,\cdot)$ is $C^{2}$ on $\R^{2d}$ for each $t$ and depends smoothly
on $t$.

\emph{Transfer of the containment assumption.}
Set $K'_{y}:=\bigcup_{s\in[0,T]}\Lambda_{s}(K')$, the image of the
compact set $[0,T]\times K'$ under the continuous map
$(s,z)\mapsto\Lambda_{s}(z)$, hence compact. For $z\in K$ and
$s\in[0,T]$, \eqref{eq:app_conjugation} gives
$\psi^{s}(z)=\Lambda_{s}(\varphi^{s}(z))$, and since $\Lambda_{s}$ is a
linear isomorphism,
\begin{equation*}
    \psi^{s}(K)\subset\Lambda_{s}\big(\operatorname{int}K'\big)
    =\operatorname{int}\Lambda_{s}(K')\subseteq\operatorname{int}K'_{y} .
\end{equation*}
Every trajectory of $\Ht$ issued from $K$ therefore stays in the fixed
compact set $K'_{y}$ on $[0,T]$. The statement of
\cite[Theorem~1]{snf} assumes forward invariance, but its proof uses
containment only in this form, through a Gronwall estimate together with
a stopping-time argument on a compact neighbourhood.

\emph{Approximation of $\psi^{t}$.}
We apply the argument of \cite[Appendix~F]{snf} to $\Ht$ on $K'_{y}$ and
recall its five steps. First, by the Weierstrass approximation theorem
$\Ht$ is replaced by a Hamiltonian polynomial in the phase-space
variable with coefficients continuous in $t$, whose flow is uniformly
close to $\psi^{t}$ on $K\times[0,T]$. Second, that flow is split into
$N$ substeps of length $t/N$. Third, each substep is approximated to
$O(t^{2}/N^{2})$ by the time-one flow of a Hamiltonian whose two parts
depend on the positions and on the momenta separately
\cite[Lemma~1]{turaev}, and the two parts are then replaced by
single-hidden-layer networks, which are dense in $C^{1}$ on compact sets
\cite{mlp2}. Fourth, a Lie--Trotter splitting turns
each such time-one flow into a composition of maps
\begin{equation}
    \big(q,\,p\big)\mapsto\big(q+\nabla_p[\,\widetilde{V}(t,p)-\widetilde{V}(0,p)\,],\;p\big),
    \qquad
    \big(q,\,p\big)\mapsto\big(q,\;p-\nabla_q[\,\widetilde{V}(t,q)-\widetilde{V}(0,q)\,]\big),
    \label{eq:app_sfl_layer}
\end{equation}
each pair contributing an error $O(1/R^{2})$ over $R$ substeps. Fifth,
the approximations are combined, the composition of such maps being
again of the same form. The conclusion is that for every
$\varepsilon_{1}>0$ there exist $M\in\mathbb{N}$ and potentials
$\widetilde{V}^{q}_{1},\ldots,\widetilde{V}^{q}_{M}$,
$\widetilde{V}^{p}_{1},\ldots,\widetilde{V}^{p}_{M}$, each $C^{2}$ in
the phase-space variable and $C^{1}$ in $t$, such that
\begin{equation}
    S^{t}:=Q_{F^{t}_{M}}\circ P_{G^{t}_{M}}\circ\cdots\circ
           Q_{F^{t}_{1}}\circ P_{G^{t}_{1}} ,
    \qquad
    \begin{aligned}
    G^{t}_{m} &:= \widetilde{V}^{q}_{m}(t,\cdot)-\widetilde{V}^{q}_{m}(0,\cdot),\\
    F^{t}_{m} &:= \widetilde{V}^{p}_{m}(t,\cdot)-\widetilde{V}^{p}_{m}(0,\cdot),
    \end{aligned}
    \label{eq:app_sympflow}
\end{equation}
satisfies
$\sup_{z\in K}\sup_{t\in[0,T]}\norm{S^{t}(z)-\psi^{t}(z)}<\varepsilon_{1}$,
with $P_{A}$ and $Q_{B}$ the gradient shears of the proof of
Theorem~1.

\emph{Normal form of the architecture.}
Besides $P_{A}\circ P_{A'}=P_{A+A'}$ and
$C_{\lambda}\circ C_{\mu}=C_{\lambda\mu}$,
\begin{equation}
    P_{A}\circ C_{\lambda}=C_{\lambda}\circ P_{A/\lambda},
    \qquad
    Q_{B}\circ C_{\lambda}=C_{\lambda}\circ Q_{B^{\lambda}},
    \qquad
    B^{\lambda}(p):=\tfrac{1}{\lambda}B(\lambda p),
    \label{eq:app_commute}
\end{equation}
the second because $\nabla B^{\lambda}(p)=\nabla B(\lambda p)$. Applying
\eqref{eq:app_commute} inside one block gives
$F_{\ell}^{t}=C_{\lambda_{\ell}^{2}}\circ P_{A_{\ell}/\lambda_{\ell}}
\circ Q_{B_{\ell}^{\lambda_{\ell}}}\circ P_{A_{\ell}/\lambda_{\ell}}$,
and moving every scaling to the left across the remaining blocks yields
\begin{equation}
    \Phih^{t}
    = C_{e^{\gamma t}}\circ
      P_{\widehat{G}_{L+1}}\circ Q_{\widehat{F}_{L}}\circ
      P_{\widehat{G}_{L}}\circ\cdots\circ
      Q_{\widehat{F}_{1}}\circ P_{\widehat{G}_{1}} ,
    \label{eq:app_normalform}
\end{equation}
the total scaling being
$\prod_{\ell}\lambda_{\ell}^{2}=e^{\gamma t}$ by
$\sum_{\ell}\alpha_\ell(\xi)=1$. Writing
$c_{\ell}:=\big(\lambda_{\ell}\prod_{j<\ell}\lambda_{j}^{2}\big)^{-1}>0$
and $\kappa_{\ell}:=\lambda_{\ell}\prod_{j<\ell}\lambda_{j}^{2}$, the
merged potentials are
\begin{equation*}
    \widehat{F}_{\ell}=\big(B_{\ell}\big)^{\kappa_{\ell}},
    \qquad
    \widehat{G}_{1}=c_{1}A_{1},
    \qquad
    \widehat{G}_{\ell}=c_{\ell-1}A_{\ell-1}+c_{\ell}A_{\ell}
    \ \ (2\le\ell\le L),
    \qquad
    \widehat{G}_{L+1}=c_{L}A_{L} .
\end{equation*}

\emph{Matching.}
Take $L=M+1$. Prescribing $\widehat{G}_{\ell}=G^{t}_{\ell}$ for
$\ell\le M$ determines $A_{1},\ldots,A_{M}$ recursively, since the
relations are triangular with diagonal entries $c_{\ell}\neq 0$.
Prescribing $\widehat{F}_{\ell}=F^{t}_{\ell}$ for $\ell\le M$ determines
$B_{\ell}$, because $(B^{\lambda})^{1/\lambda}=B$; set
$B_{M+1}=0$, so $Q_{\widehat{F}_{M+1}}=\mathrm{id}$ and the last two
kicks of \eqref{eq:app_normalform} merge into
$P_{\widehat{G}_{M+1}+\widehat{G}_{M+2}}$. Choosing
\begin{equation*}
    A_{M+1}=-\frac{c_{M}A_{M}}{c_{M+1}+c_{M+2}}
\end{equation*}
makes this merged kick the identity, and \eqref{eq:app_normalform}
becomes $C_{e^{\gamma t}}\circ S^{t}$.

\emph{Admissibility of the potentials.}
By the abbreviations introduced in the proof of Theorem~1, the block potentials are
$\At_{\ell}(\,\cdot\,;t,\xi)=\tfrac{2}{t}A_{\ell}$ and
$\Bt_{\ell}(\,\cdot\,;t,\xi)=\tfrac{1}{t}B_{\ell}$. Each $A_{\ell}$ and
$B_{\ell}$ is a fixed linear combination of the differences
\eqref{eq:app_sympflow}, and for $\widetilde{V}$ of class $C^{1}$ in $t$
\begin{equation*}
    \frac{\widetilde{V}(t,\cdot)-\widetilde{V}(0,\cdot)}{t}
    \xrightarrow[t\to 0]{} \partial_{t}\widetilde{V}(0,\cdot) ,
\end{equation*}
so $\At_{\ell}$ and $\Bt_{\ell}$ extend continuously to $t=0$ and are
$C^{2}$ in the phase-space variable for every $t\in[0,T]$. Set
$a_{\ell}=b_{\ell}=0$, so that $\At_{\ell}=\delta A_{\ell}$ and
$\Bt_{\ell}=\delta B_{\ell}$. These networks are dense in $C^{1}$ on
compact subsets of $\R^{d}\times[0,T]$
\cite{mlp,mlp2}, so for every
$\eta>0$ there are parameters with
$\norm{\nabla\delta A_{\ell}-\nabla A_{\ell}}_{\infty}\le\eta$ and
$\norm{\nabla\delta B_{\ell}-\nabla B_{\ell}}_{\infty}\le\eta$ on the
relevant compact sets. Each factor of \eqref{eq:app_normalform} is
Lipschitz there, so composing $3L+1$ of them gives a constant $c>0$,
depending only on those Lipschitz constants, with
\begin{equation*}
    \sup_{z\in K}\sup_{t\in[0,T]}
    \norm{\Phih^{t}(z)-C_{e^{\gamma t}}\big(S^{t}(z)\big)}\le c\,\eta .
\end{equation*}

\emph{Conclusion.}
Since $\gamma<0$, the map $C_{e^{\gamma t}}$ is Lipschitz with constant
$e^{\gamma t}\le 1$, so by \eqref{eq:app_conjugation}
\begin{equation*}
    \norm{\Phih^{t}(z)-\varphi^{t}(z)}
    \le \norm{\Phih^{t}(z)-C_{e^{\gamma t}}\big(S^{t}(z)\big)}
      + \norm{C_{e^{\gamma t}}\big(S^{t}(z)\big)
              -C_{e^{\gamma t}}\big(\psi^{t}(z)\big)}
    \le c\,\eta+\varepsilon_{1} .
\end{equation*}
Taking $\varepsilon_{1}=\varepsilon/2$ and $\eta=\varepsilon/(2c)$ gives
the stated bound with $L=M+1$ and parameters $\theta_{L}$ that do not
depend on $t$. The potentials are $C^{2}$ in the phase-space variable
for every $t\in[0,T]$ and the allocations satisfy
$\sum_{\ell}\alpha_\ell(\xi)=1$, so Theorem~1 gives
$\big(\Phih^{t}\big)^{*}\omega=e^{\gamma t}\omega$.
\end{proof}

\section{Experimental Details}


\subsection{Data Generation}

The experiments use two-degree-of-freedom natural Hamiltonians
$H(q,p)=\tfrac12\lVert p\rVert^{2}+V(q)$ on the position domain
$\Q=[-3,3]^{2}$, so that $H_q=V$ and $H_p=\tfrac12\lVert p\rVert^{2}$
in the block potentials of the main paper. Each potential is
$V=V_{\mathrm{q}}+V_{\mathrm{GP}}$ with a quartic backbone
\begin{equation*}
    V_{\mathrm{q}}(q)
    =\tfrac12\big(b_1q_1^{2}+b_2q_2^{2}\big)+c\,q_1q_2
    +\tfrac14\big(a_1q_1^{4}+a_2q_2^{4}+2a_3q_1^{2}q_2^{2}\big) .
\end{equation*}
The coefficients are drawn independently with $b_1,b_2\sim U[-1,2.25]$,
$c\sim U[-0.5,0.5]$ and $a_i=0.5\,u_i^{2}$ with $u_i\sim U[0,1]$, so the
quartic coefficients are concentrated near zero. Negative $b_i$ are
admitted, so the family contains single well, double well and multiple
well landscapes.

The perturbation is a Gaussian process with a squared exponential kernel,
realized with $D=128$ random Fourier features,
\begin{equation*}
    V_{\mathrm{GP}}(q)
    = \beta\sqrt{\tfrac{2}{D}}\sum_{k=1}^{D}
      a_k\cos\big(w_k^{\top}q+\phi_k\big),
    \qquad
    w_k\sim\mathcal{N}(0,I_2)/\ell,
    \quad \phi_k\sim U[0,2\pi],
    \quad a_k\sim\mathcal{N}(0,1),
\end{equation*}
with length scale $\ell\sim U[0.6,1.8]$. The amplitude $\beta$ is set so
that
\begin{equation*}
    \beta\,\max_{q\in\mathcal{G}}\lVert\nabla V_{\mathrm{GP}}^{\,0}(q)\rVert
    = \delta\cdot\operatorname{median}_{q\in\mathcal{G}}
      \lVert\nabla V_{\mathrm{q}}(q)\rVert ,
\end{equation*}
where $V_{\mathrm{GP}}^{\,0}$ is the unscaled feature sum and
$\mathcal{G}$ the descriptor grid. The parameter $\delta$ is therefore
the ratio of the largest perturbation force to the median backbone force,
and it is drawn log-uniformly with $\delta\in[0.2,1.5]$. Since every
training system has $\delta\geq0.2$, the benchmark systems, which have
$\delta=0$, lie outside the training distribution. The length scale
satisfies $\ell\geq0.6$, which is $3.2$ times the descriptor spacing
$6/32=0.1875$, so the sampled potentials are resolved by the descriptor.
The dissipation parameter is drawn log-uniformly with
$\lvert\gamma\rvert\in[0.1,0.4]$.

The descriptor is the vector of values of $V$ on a uniform $32\times32$
grid over $\Q$. In the notation of the main paper this corresponds to
sensor points $r_i=(q_i,0)$ with $q_i$ on that grid, for which
$H(r_i)=V(q_i)$.

A proposed system is accepted only if
$V_{\partial}-V_{\min}\geq0.5$, where $V_{\min}$ is the minimum of $V$
over the grid and $V_{\partial}$ its minimum on the grid boundary. This
rejection removes systems whose well is too shallow to confine a
trajectory. The energy budget is then
$E_{\max}=V_{\min}+0.9\,(V_{\partial}-V_{\min})$.

Initial conditions are drawn by rejection. Positions are uniform on $\Q$,
momenta are $p=r\,u\,p_{\max}$ with $u$ uniform on the unit circle,
$r=v^{1/2}$ for $v\sim U[0,1]$ and
$p_{\max}=\sqrt{2(E_{\max}-V_{\min})}$, and a draw is kept when
$H(q,p)\leq E_{\max}$. The pool is oversampled by a factor of $32$ and
doubled up to five times if a system does not reach its quota. Along solutions the energy is nonincreasing, so a trajectory starting
below $E_{\max}$ cannot reach a part of the domain where $V$ exceeds
$E_{\max}$. Since $V_{\partial}$ is evaluated on the boundary of the
descriptor grid rather than on $\partial\Q$, this was checked
numerically on $40$ validation systems, comprising $1280$ trajectories:
no trajectory left $\Q$, none exceeded its energy budget by more than
$10^{-6}$, and the energy was nonincreasing along every stored
trajectory to the same tolerance.

Reference trajectories are integrated with a fourth-order Runge--Kutta
scheme in double precision, using $1024$ substeps on the fixed interval
$[0,T]$ with $T=4$, and each trajectory is stored at $65$ equally spaced
states. Initial conditions are rounded to single precision before
integration, so the stored initial state coincides with the first stored
trajectory state.

The generator draws $40{,}000$ systems from a single seed,
of which $1{,}000$ are assigned to validation and $1{,}000$ to test,
leaving $38{,}000$ for training, with $32$ trajectories per system. The
system-defining parameters $w,\phi,a,\beta$ and the backbone
coefficients are stored in double precision, so the analytic potential
is reconstructed exactly at training and evaluation time; trajectories
are stored in single precision.

Training pairs $(z,\tau,z^{+})$ are formed online, where $z$ and $z^{+}$
are two stored states of the same trajectory and $\tau$ is the elapsed
time between them. Each step draws $512$ systems and $16$ pairs per
system, for a batch of $8192$. The time gaps are stratified: the index
range $[1,65)$ is divided into $16$ contiguous bins and one gap is drawn
from each bin per system, so every batch contains the full range of
query times.

This construction is used in place of an existing collection because the
setting requires the Hamiltonian and the dissipation parameter of every
system to be known exactly, both entering the architecture directly; a
family in which the distance from the training distribution is
controlled, which is what places the benchmark systems outside it; and
reference solutions accurate enough that the structure error can be
measured at the level of double precision round-off. Published
Hamiltonian trajectory datasets provide trajectories for a small number
of fixed systems and do not supply these three properties.

\subsection{Sampling Ranges and Settings}

\begin{table}[htbp]
\centering
\caption{Sampling ranges, integration settings and hyperparameters.}
\label{tab:settings}
\small
\begin{tabular}{llc}
\toprule
& Quantity & Value \\
\midrule
\multicolumn{3}{l}{\textit{Domain and descriptor}}\\
& Position domain $\Q$ & $[-3,3]^{2}$ \\
& Descriptor grid & $32\times32$ \\
& Grid spacing & $0.1875$ \\
\midrule
\multicolumn{3}{l}{\textit{System family}}\\
& $b_1,b_2$ & $U[-1,2.25]$ \\
& $c$ & $U[-0.5,0.5]$ \\
& $a_1,a_2,a_3$ & $0.5\,U[0,1]^{2}$ \\
& Random Fourier features $D$ & $128$ \\
& Length scale $\ell$ & $U[0.6,1.8]$ \\
& Perturbation ratio $\delta$ & $\log U[0.2,1.5]$ \\
& Dissipation $\lvert\gamma\rvert$ & $\log U[0.1,0.4]$ \\
& Acceptance threshold $V_{\partial}-V_{\min}$ & $\geq 0.5$ \\
& Energy budget fraction & $0.9$ \\
\midrule
\multicolumn{3}{l}{\textit{Integration}}\\
& Scheme & RK4, double precision \\
& Horizon $T$ & $4$ \\
& Substeps & $1024$ \\
& Step size $h$ & $3.906\times10^{-3}$ \\
& Stored states per trajectory & $65$ \\
\midrule
\multicolumn{3}{l}{\textit{Dataset}}\\
& Systems (train / val / test) & $38{,}000$ / $1{,}000$ / $1{,}000$ \\
& Trajectories per system & $32$ \\
& Generator seed & $20260720$ \\
\midrule
\multicolumn{3}{l}{\textit{Optimization}}\\
& Steps & {$80{,}000$} \\
& Systems per step & $512$ \\
& Pairs per system & $16$ \\
& Batch size & $8192$ \\
& Optimizer & AdamW \\
& Learning rate & $10^{-4}$, cosine to $10^{-6}$ \\
& Warmup steps & $3000$ \\
& Weight decay & $10^{-4}$ \\
& Gradient clipping & $0.5$ \\
& Validation interval & $1000$ steps \\
& Validation systems & $512$ \\
& Training seeds & $0$, $1$, $2$ \\
\midrule
\multicolumn{3}{l}{\textit{Architecture}}\\
& Blocks $L$ & $12$ \\
& FNO width / modes / layers & $32$ / $8$ / $4$ \\
& Condition dimension & $128$ \\
& Correction network width & $64$ \\
& Softmax logit bound & $8$ \\
& Soft-penalty weight $\lambda_{\mathrm{soft}}$ & $1$ \\
\midrule
\multicolumn{3}{l}{\textit{Long-horizon evaluation}}\\
& $\lvert\gamma\rvert$ & $0.1$ \\
& Windows & $10$ \\
& Query times per window & $32$ \\
& Initial conditions & $64$ \\
& Reference RK4 steps per unit time & $512$ \\
\bottomrule
\end{tabular}

\end{table}

The final checkpoint is the one with the lowest validation loss. No
hyperparameter search was carried out; the values above were fixed
before the reported runs.

\subsection{Benchmark Systems}

The four benchmark systems have $V=V_{\mathrm{q}}$ with the
coefficients of Table~\ref{tab:bench}, that is, no Gaussian process
perturbation. Initial conditions are drawn from each system's own
energy sublevel set by the same rejection scheme as the training data,
with $E_{\max}$ computed on the native $32\times32$ grid. Reference
trajectories use RK4 in double precision with step size approximately
$10^{-2}$, which agrees with an adaptive high-order reference to about
$10^{-9}$ on all four systems.

\begin{table}[H]
\centering
\caption{Backbone coefficients of the four benchmark systems.}
\label{tab:bench}

\small
\begin{tabular}{llcccccc}
\toprule
& System & $b_1$ & $b_2$ & $c$ & $a_1$ & $a_2$ & $a_3$ \\
\midrule
B1 & Anisotropic oscillator & $0.25$ & $1.96$ & $0$ & $0$ & $0$ & $0$ \\
B2 & Coupled Duffing        & $-0.5$ & $-0.5$ & $0.3$ & $0.25$ & $0.25$ & $0$ \\
B3 & Mexican hat            & $-0.5$ & $-0.5$ & $0$ & $0.25$ & $0.25$ & $0.5$ \\
B4 & Quartic coupled        & $1.0$ & $1.0$ & $0$ & $0$ & $0$ & $0.5$ \\
\bottomrule
\end{tabular}

\end{table}

Figure~\ref{fig:bench_potentials} shows the resulting potentials. 

\begin{figure}[H]
    \centering
    \includegraphics[width=\linewidth]{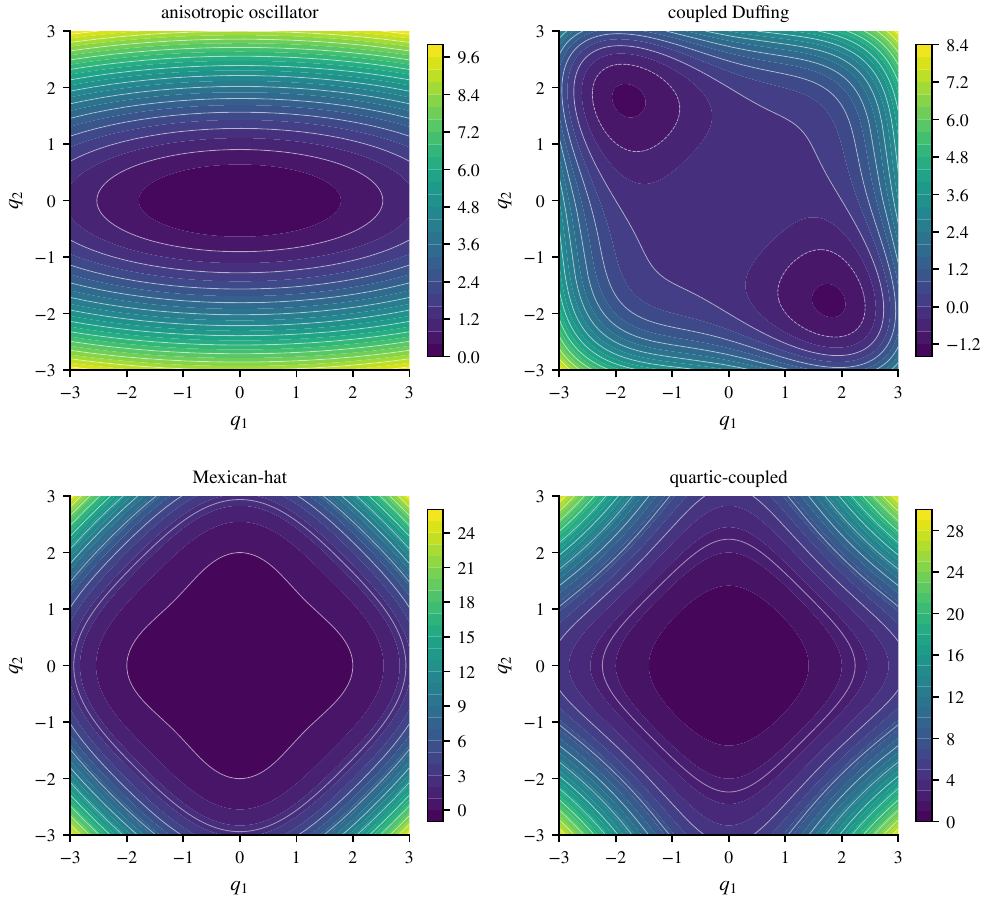}
    \caption{
    Potentials $V$ of the four benchmark systems on
    $[-3,3]^{2}$. 
    }
\label{fig:bench_potentials}
\end{figure}

\subsection{Computing Environment}

All models were trained and evaluated on one MIG partition of profile
\texttt{1g.12gb} of an NVIDIA H100, providing approximately $11.5$\,GB
of device memory. The software stack is PyTorch 2.3.1 with CUDA 12.1 on
Red Hat Enterprise Linux 8. Latency figures reported elsewhere in this
document were measured on the same partition and are not comparable
with full-card measurements.

\section{Additional Results}
\subsection{Per-System Energy Error}

Table~\ref{tab:long_energy} gives the per-system energy errors
summarized by the last column of the main paper. DeepONet diverges on
B3, which is why its mean is dominated by that system.

\begin{table}[htbp]
\centering
\caption{
Long-horizon energy error $\E_{H}$ on the four benchmark systems,
averaged over $t\in(0,10T]$ and over $64$ initial conditions at
$\lvert\gamma\rvert=0.1$, under the protocol of the main paper.
}
\label{tab:long_energy}

\setlength{\tabcolsep}{4pt}
\small
\begin{tabular}{lccccc}
\toprule
Model & B1 & B2 & B3 & B4 & Mean \\
\midrule
\multicolumn{6}{l}{\textit{Baselines}} \\
Symplectic Flow & $5.4_{\pm0.0}{\times}10^{-1}$ & $5.3_{\pm0.1}{\times}10^{-1}$ & $4.8_{\pm0.0}{\times}10^{-1}$ & $4.8_{\pm0.0}{\times}10^{-1}$ & $5.1_{\pm0.0}{\times}10^{-1}$ \\
MLP Flow & $2.0_{\pm0.4}{\times}10^{-2}$ & $4.6_{\pm0.1}{\times}10^{-2}$ & $2.3_{\pm0.0}{\times}10^{-2}$ & $2.0_{\pm0.1}{\times}10^{-2}$ & $2.7_{\pm0.1}{\times}10^{-2}$ \\
MLP Flow (soft) & $1.7_{\pm0.2}{\times}10^{-2}$ & $4.4_{\pm0.8}{\times}10^{-2}$ & $2.6_{\pm0.1}{\times}10^{-2}$ & $1.7_{\pm0.2}{\times}10^{-2}$ & $2.6_{\pm0.2}{\times}10^{-2}$ \\
DeepONet & $3.2_{\pm0.6}{\times}10^{-2}$ & $9.6_{\pm0.4}{\times}10^{-2}$ & $4.6_{\pm8.0}{\times}10^{6}$ & $4.0_{\pm0.5}{\times}10^{-2}$ & $1.2_{\pm2.0}{\times}10^{6}$ \\
Transolver & $1.4_{\pm0.1}{\times}10^{-2}$ & $4.2_{\pm0.5}{\times}10^{-2}$ & $1.9_{\pm0.1}{\times}10^{-2}$ & $1.4_{\pm0.1}{\times}10^{-2}$ & $2.2_{\pm0.2}{\times}10^{-2}$ \\
\midrule
\multicolumn{6}{l}{\textit{Ablations}} \\
\quad Fixed Allocation & $2.2_{\pm0.5}{\times}10^{-3}$ & $2.8_{\pm0.5}{\times}10^{-3}$ & $2.3_{\pm0.5}{\times}10^{-3}$ & $1.6_{\pm0.2}{\times}10^{-3}$ & $2.2_{\pm0.4}{\times}10^{-3}$ \\
\quad No Correction & $2.1_{\pm0.6}{\times}10^{-3}$ & $3.8_{\pm1.0}{\times}10^{-3}$ & $2.6_{\pm0.2}{\times}10^{-3}$ & $2.2_{\pm0.3}{\times}10^{-3}$ & $2.7_{\pm0.5}{\times}10^{-3}$ \\
\midrule
\textbf{CoSynFlow} & $\bm{1.3_{\pm0.1}{\times}10^{-3}}$ & $\bm{2.5_{\pm0.3}{\times}10^{-3}}$ & $\bm{2.1_{\pm0.3}{\times}10^{-3}}$ & $\bm{1.2_{\pm0.2}{\times}10^{-3}}$ & $\bm{1.8_{\pm0.1}{\times}10^{-3}}$ \\
\bottomrule
\end{tabular}

\end{table}

\subsection{Energy Error over the Long Horizon}

The energy errors of Table~\ref{tab:long_energy} are averages over
$t\in(0,10T]$. Figure~\ref{fig:energy_baseline} and
Figure~\ref{fig:energy_ablation} resolve them in time, under the same
protocol and at $\lvert\gamma\rvert=0.1$.

\begin{figure}[htbp]
    \centering
    \begin{subfigure}{0.45\linewidth}
        \includegraphics[width=\linewidth]{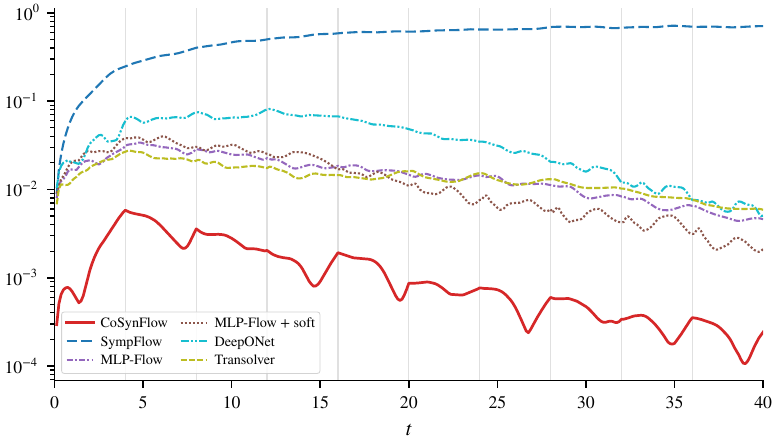}
        \caption{B1, anisotropic oscillator}
    \end{subfigure}
    \hfill
    \begin{subfigure}{0.45\linewidth}
        \includegraphics[width=\linewidth]{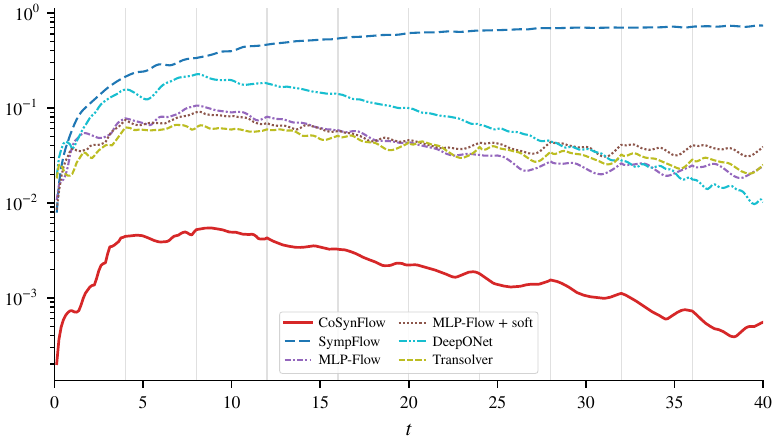}
        \caption{B2, coupled Duffing}
    \end{subfigure}

    \vspace{6pt}

    \begin{subfigure}{0.45\linewidth}
        \includegraphics[width=\linewidth]{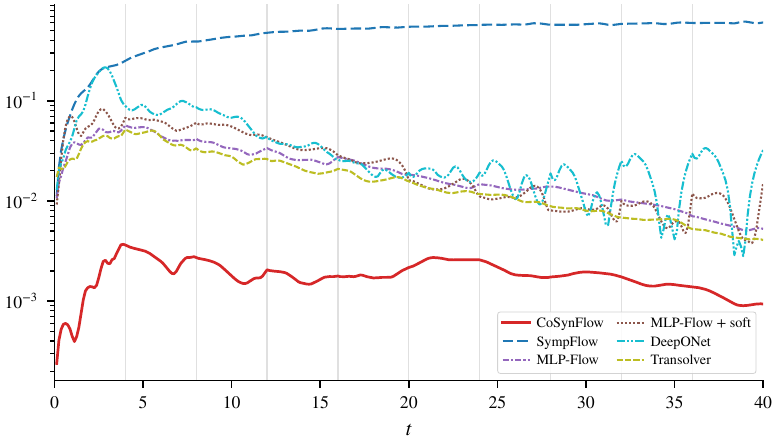}
        \caption{B3, Mexican hat}
    \end{subfigure}
    \hfill
    \begin{subfigure}{0.45\linewidth}
        \includegraphics[width=\linewidth]{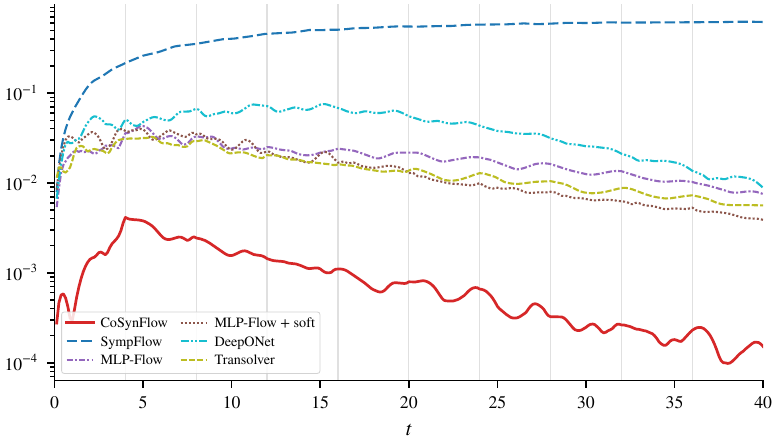}
        \caption{B4, quartic coupled}
    \end{subfigure}
    \caption{
    Energy error $\E_{H}(t)$ for CoSynFlow and the five baselines over
    $0\leq t\leq 10T$ at $\lvert\gamma\rvert=0.1$, averaged over $64$
    initial conditions, on a logarithmic scale.
    }
    \label{fig:energy_baseline}
\end{figure}

\begin{figure}[htbp]
    \centering
    \begin{subfigure}{0.45\linewidth}
        \includegraphics[width=\linewidth]{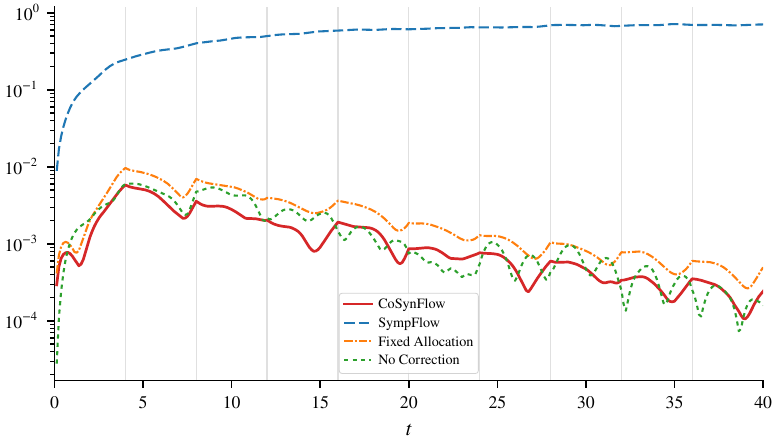}
        \caption{B1, anisotropic oscillator}
    \end{subfigure}
    \hfill
    \begin{subfigure}{0.45\linewidth}
        \includegraphics[width=\linewidth]{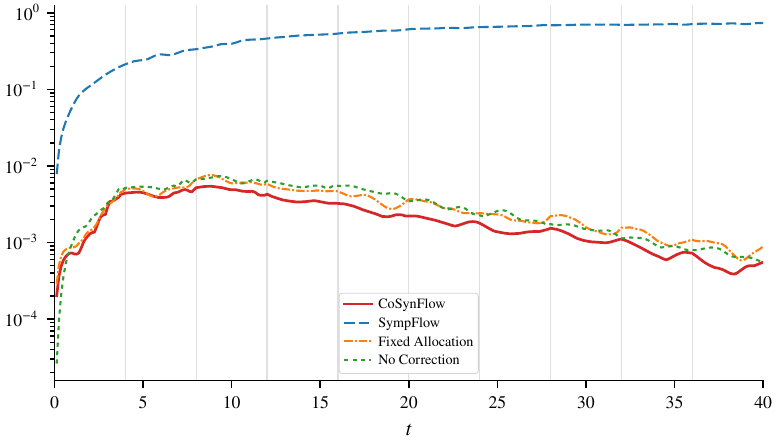}
        \caption{B2, coupled Duffing}
    \end{subfigure}

    \vspace{6pt}

    \begin{subfigure}{0.45\linewidth}
        \includegraphics[width=\linewidth]{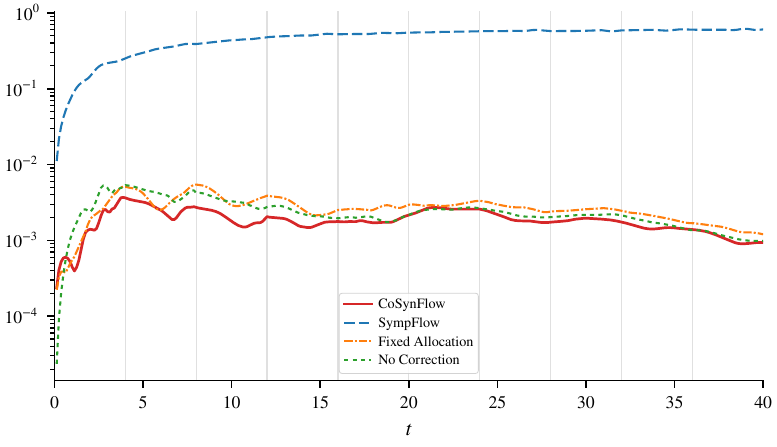}
        \caption{B3, Mexican hat}
    \end{subfigure}
    \hfill
    \begin{subfigure}{0.45\linewidth}
        \includegraphics[width=\linewidth]{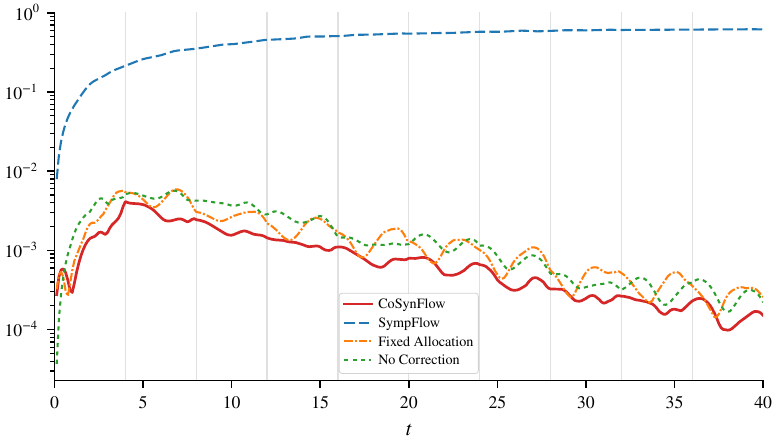}
        \caption{B4, quartic coupled}
    \end{subfigure}
    \caption{
    Energy error $\E_{H}(t)$ for CoSynFlow, Symplectic Flow and the two
    ablations, under the protocol of Figure~\ref{fig:energy_baseline}.
    }
    \label{fig:energy_ablation}
\end{figure}

In the position plane and in the time series of $q_1$, the three
structured models are visually indistinguishable over the whole horizon,
so those plots are not reproduced here for the ablation group. Their
differences appear only in the aggregate errors of
Table~\ref{tab:long_energy} and of the main paper.

\subsection{Structure Error under Composition}

The main paper reports the structure error at $t=2$ and after composing
the learned map sixteen times, a total time of $7.5T$.
Table~\ref{tab:compose} gives both.

\begin{table}[htbp]
\centering
\caption{
Conformal symplectic structure error, evaluated in double precision at
$\gamma=-0.3$.
}
\label{tab:compose}
\setlength{\tabcolsep}{6pt}
\small
\begin{tabular}{lcc}
\toprule
Model & $\E_\omega$ at $t=2$ & $\E_\omega$ after $16$ compositions \\
\midrule
\multicolumn{3}{l}{\textit{Baselines}} \\
Symplectic Flow & $8.2{\times}10^{-1}$ & $8.1{\times}10^{3}$ \\
MLP Flow        & $2.6{\times}10^{-1}$ & $6.2{\times}10^{1}$ \\
MLP Flow (soft) & $1.3{\times}10^{-1}$ & $2.9{\times}10^{1}$ \\
DeepONet        & $7.4{\times}10^{-1}$ & $3.9{\times}10^{2}$ \\
Transolver      & $2.2{\times}10^{-1}$ & $2.3{\times}10^{1}$ \\
\midrule
\multicolumn{3}{l}{\textit{Ablations}} \\
\quad Fixed Allocation & $1.6{\times}10^{-15}$ & $8.8{\times}10^{-14}$ \\
\quad No Correction    & $9.9{\times}10^{-16}$ & $5.9{\times}10^{-14}$ \\
\midrule
\textbf{CoSynFlow}     & $1.3{\times}10^{-15}$ & $1.2{\times}10^{-13}$ \\
\bottomrule
\end{tabular}

\end{table}

\subsection{Physics-Informed Training}

Table~\ref{tab:pi_full} gives the full results of the physics-informed
experiment. The model is CoSynFlow instantiated for the single system
B2 at $\lvert\gamma\rvert=0.2$, trained on two reference trajectories
for $4000$ steps with $\lambda=1$, with three seeds. Errors are measured
on $64$ held-out trajectories, within the trained horizon and after
composition to $10T$.

\begin{table}[htbp]
\centering
\caption{
Physics-informed training of CoSynFlow on a single system with two
training trajectories, three seeds. The last column is the ratio of the
two means.
}
\label{tab:pi_full}
\setlength{\tabcolsep}{5pt}
\small
\begin{tabular}{lccc}
\toprule
& $\mathcal{L}_{\mathrm{pred}}$ & $\mathcal{L}_{\mathrm{pred}}+\lambda\mathcal{L}_{\mathrm{PI}}$ & Ratio \\
\midrule
State error, $t\in(0,T]$   & $1.21_{\pm0.54}{\times}10^{-2}$ & $\bm{2.95_{\pm0.17}{\times}10^{-3}}$ & $4.09$ \\
State error, $t\in(0,10T]$ & $4.61_{\pm1.5\phantom{0}}{\times}10^{-2}$ & $\bm{9.88_{\pm0.55}{\times}10^{-3}}$ & $4.66$ \\
Energy error, $t\in(0,T]$  & $2.95_{\pm1.7\phantom{0}}{\times}10^{-3}$ & $\bm{4.94_{\pm0.41}{\times}10^{-4}}$ & $5.97$ \\
\midrule
Structure error            & $7.5{\times}10^{-16}$ & $6.9{\times}10^{-16}$ & --- \\
\bottomrule
\end{tabular}

\end{table}

For each of the three metrics, every seed trained with the residual term
achieves a lower error than every seed trained without it. The relative
standard deviation across seeds falls from $45\%$ to $5.8\%$ for the
state error within the trained horizon.

\bibliographystyle{unsrt}  

\bibliography{main}

\end{document}